\documentclass{article}
\usepackage{amsmath,amsfonts,bm}

\def\1{\bm{1}}

\def\rvb{{\mathbf{b}}}
\def\rvc{{\mathbf{c}}}

\def\rvu{{\mathbf{i}}}

\def\rvk{{\mathbf{k}}}

\def\rvu{{\mathbf{u}}}
\def\rvv{{\mathbf{v}}}

\def\rvx{{\mathbf{x}}}
\def\rvy{{\mathbf{y}}}
\def\rvz{{\mathbf{z}}}

\def\vmu{{\bm{\mu}}}

\def\vtheta{{\bm{\theta}}}
\def\vepsilon{{\bm{\epsilon}}}

\def\vmu{{\bm{\mu}}}

\def\ve{{\bm{e}}}

\def\vm{{\bm{m}}}

\def\vs{{\bm{s}}}

\def\vw{{\bm{w}}}
\def\vx{{\bm{x}}}

\def\mD{{\bm{D}}}

\def\mI{{\bm{I}}}

\DeclareMathAlphabet{\mathsfit}{\encodingdefault}{\sfdefault}{m}{sl}
\SetMathAlphabet{\mathsfit}{bold}{\encodingdefault}{\sfdefault}{bx}{n}

\newcommand{\pdata}{q_{\rm{data}}}
\newcommand{\E}{\mathop{\mathbb{E}}}

\newcommand{\R}{\mathbb{R}}

\DeclareMathOperator*{\argmin}{arg\,min}

\newcommand{\sender}{q_{_S}}
\newcommand{\out}{p_{_O}}
\newcommand{\rec}{p_{_R}}
\newcommand{\inp}{q_{_I}}
\newcommand{\flow}{q_{_F}}
\newcommand{\update}{q}

\def\vxi{{\bm{\xi}}}

\newcommand*{\dif}{\mathop{}\!\mathrm{d}}

\usepackage{microtype}
\usepackage{graphicx}
\usepackage{subfigure}
\usepackage{multirow}
\usepackage{booktabs} % for professional tables

\usepackage{hyperref}

\usepackage[accepted]{icml2024}

\usepackage{amsmath}
\usepackage{amssymb}
\usepackage{mathtools}
\usepackage{amsthm}
\usepackage{breqn}

\usepackage[capitalize,noabbrev]{cleveref}

\theoremstyle{plain}
\newtheorem{theorem}{Theorem}[section]
\newtheorem{proposition}[theorem]{Proposition}
\newtheorem{lemma}[theorem]{Lemma}

\theoremstyle{definition}

\theoremstyle{remark}

\usepackage[textsize=tiny]{todonotes}

\newcommand{\cx}[1]{{\color{red}{[\textbf{cx: }#1]}}}

\icmltitlerunning{Unifying Bayesian Flow Networks and Diffusion Models through
Stochastic Differential Equations}

\begin{document}

\twocolumn[
\icmltitle{Unifying Bayesian Flow Networks and Diffusion Models through \\
Stochastic Differential Equations}



\icmlsetsymbol{equal}{*}
% \icmlsetsymbol{cor}{$\dagger$}

\begin{icmlauthorlist}
\icmlauthor{Kaiwen Xue}{equal,renmin}
\icmlauthor{Yuhao Zhou}{equal,thu}
\icmlauthor{Shen Nie}{renmin}
\icmlauthor{Xu Min}{ant}

\icmlauthor{Xiaolu Zhang}{ant}
\icmlauthor{Jun Zhou}{ant}
\icmlauthor{Chongxuan Li}{renmin}
\end{icmlauthorlist}

\icmlaffiliation{renmin}{Gaoling School of AI, Renmin University of China, Beijing, China}
\icmlaffiliation{thu}{Department of Computer Science and Technology, Tsinghua University, Beijing, China}
\icmlaffiliation{ant}{Ant Group, Hangzhou, China}

\icmlcorrespondingauthor{Chongxuan Li}{chongxuanli@ruc.edu.cn}

\newcommand{\fix}{\marginpar{FIX}}
\newcommand{\new}{\marginpar{NEW}}

\icmlkeywords{Machine Learning, ICML}

\vskip 0.3in
]


% This command .actually creates the footnote in the first column

\printAffiliationsAndNotice{\icmlEqualContribution} % otherwise use the standard text.

\begin{abstract}
Bayesian flow networks (BFNs) iteratively refine the parameters, instead of the samples in diffusion models (DMs), of distributions at various noise levels through Bayesian inference. Owing to its differentiable nature, BFNs are promising in modeling both continuous and discrete data, while simultaneously maintaining fast sampling capabilities. This paper aims to understand and enhance BFNs by connecting them with DMs through stochastic differential equations (SDEs). We identify the linear SDEs corresponding to the noise-addition processes in BFNs, demonstrate that BFN's regression losses are aligned with denoise score matching, and validate the sampler in BFN as a first-order solver for the respective reverse-time SDE. 
Based on these findings and existing recipes of fast sampling in DMs, we propose specialized solvers for BFNs that markedly surpass the original BFN sampler in terms of sample quality with a limited number of function evaluations (e.g., 10) on both image and text datasets. Notably, our best sampler achieves an increase in speed of $5\sim20$ times for free. Our code is available at \url{https://github.com/ML-GSAI/BFN-Solver}.
\end{abstract}

\section{Introduction}

\begin{table*}[t]
\caption{Technical contributions of the paper include the theory on unifying BFN and DM (in red) and new samplers for BFN inspired by the theory (in blue). ``SDE-solver1'' means a first-order solver for the corresponding SDE and ``Approx.'' is a shorthand for ``Approximate''.}
\label{tab:contribution}
\vspace{.1cm}
\begin{center}
\begin{small}
\begin{sc}
\begin{tabular}{lcccccc}
\toprule
 & Noise-adding process & Loss function & Original sampler   & New samplers \\
\midrule
BFN on con-  & Corresponding SDE & Equivalent to DSM & SDE-solver1 & BFN-Solvers\\
tinuous data     & {\color{red} Theorem~\ref{thm:cbfn-sde}}  & Trivial & {\color{red} Proposition~\ref{prop:convergence}} & {\color{blue} Algos.~1-3 in Appendix} \\
\midrule
BFN on dis-    &  Corresponding SDE & Equivalent to DSM & Approx. SDE-solver1 & BFN-Solvers \\
 create data   & {\color{red} Theorem~\ref{thm:dbfn-sde}}  & {\color{red} Theorem~\ref{thm:discrete-dsm} } & {\color{red} Proposition~\ref{thm:dbfn-sample}} & {\color{blue} Algos.~4-7 in Appendix} \\
\bottomrule
\end{tabular}
\end{sc}
\end{small}
\end{center}
\vskip -0.1in
\end{table*}

Deep generative models (DGMs) are effective in capturing complex data distributions and producing realistic samples, substantially influencing fields such as computer vision~\cite{rombach2022high,ramesh2022hierarchical,podell2023sdxl} and natural language processing~\cite{brown2020language,gpt4}. The fundamental challenge in DGMs is to represent a flexible probability distribution that facilitates effective parameter learning and efficient inference simultaneously, greatly depending on the data (or modality).

Autoregressive models (ARMs)~\cite{gpt4}, for example, excel in modeling sequential and discrete data (e.g., text) but face limitations in the inference speed, which is proportional to the number of variables. Diffusion models (DMs)~\cite{sohl2015deep,ho2020denoising,song2021scorebased}, on the other hand, better balance generation quality and efficiency with a coarse-to-fine approach. Although considered state-of-the-art in image generation, DMs encounter challenges in handling discrete variables, where score matching algorithms~\cite{hyvarinen2005estimation,vincent2011connection} do not directly apply.

A new class of generative models, Bayesian Flow Networks (BFNs)~\cite{graves2023bayesian}, has been developed recently to overcome these challenges. While inspired by DMs, BFNs distinguish themselves by focusing on iteratively refining the parameters (instead of the samples) of a distribution set at different noise levels through Bayesian inference (see Sec.~\ref{sec:background} for more details). This strategy enables BFNs to facilitate fast sampling and maintain a continuous nature, even when processing discrete data. With carefully designed regression losses, BFNs have shown considerable promise in both image and language modeling. Notably, BFN is primarily developed based on a message-sending process with minimum communication length, and the exact relation between BFNs and DMs remains unclear. 

As summarized in Table~\ref{tab:contribution},  
this paper primarily contributes by unifying BFNs and DMs through stochastic differential equations (SDEs), a pivotal step in understanding their relationship and enhancing BFNs. Initially, by slightly truncating the time, we identify linear SDEs corresponding to the noise-adding processes in BFN on both continuous (see Sec.~\ref{sec:CBCD}) and discrete data (see Sec.~\ref{sec:CBDD}) and derive the reverse-time SDEs for sampling. Note that the SDEs for discrete data operate on a set of latent variables, which the original BFN formulation marginalizes out, rather than distribution parameters.
Furthermore, we demonstrate that, especially on discrete data, BFN's regression losses align with denoising score matching (DSM)~\cite{vincent2011connection} w.r.t. variables in the corresponding SDE, positioning the trained networks to naturally parameterize the reverse-time SDEs. Finally, the original BFN sampler is proven as an (approximate) first-order solver for the corresponding reverse-time SDE. 

The explicit connection between BFNs and DMs brings immediate benefits, particularly in applying fast sampling methods~\cite{lu2022dpm,lu2022dpm++} from DMs to BFNs. We derive the corresponding probability flow ordinary differential equations (ODEs)~\cite{song2021scorebased} for BFNs on both continuous and discrete data. We propose high-order solvers  (named \emph{BFN-Solvers}) tailored to BFNs' special (e.g., semi-linear) structure, for both SDEs and ODEs. Empirically, using the same pre-trained model, our best solver significantly outperforms the original BFN sampler with a few (e.g., $10$) number of function evaluations (NFE) under sample quality on both the CIFAR10 and text8 datasets, achieving a $5\sim20$ times increase in speed for free (see Sec.~\ref{sec:exp} for details).

We believe our discovery offers a rigorous and systematic perspective for analyzing and improving the training and inference processes of BFNs, grounded in the existing results of DMs, and may inspire future work as detailed in Sec.~\ref{sec:conclusion}.

\section{Related Work}
\label{sec:related}

\textbf{Score-baesd DMs.} 
Built upon the score matching algorithms~\cite{hyvarinen2005estimation,vincent2011connection,song2019sliced,pang2020efficient}, DMs~\citep{sohl2015deep,ho2020denoising,song2021scorebased} are currently SOTA to model continuous variables~\citep{dhariwal2021diffusion,chen2020wavegrad, kong2020diffwave,ho2022imagen, singer2022make,poole2022dreamfusion, wang2023prolificdreamer}. In particular, large-scale text-to-image models~\citep{rombach2022high, ramesh2022hierarchical,saharia2022photorealistic, bao2023one,balaji2023ediffi, xue2023raphael,podell2023sdxl} have made remarkable progress and attracted significant attention.   

\textbf{Solvers for DMs.} Since~\citet{song2021scorebased} introduced the SDE and probability flow ODE formulation of DMs, there have been extensive solvers for both SDE~\cite{ho2020denoising, song2021scorebased, karras2022elucidating, lu2022dpm++, bao2022analytic, bao2022estimating, jolicoeur2021gotta, xue2023sa,guo2023gaussian} and ODE~\cite{song2020denoising,liu2022pseudo, lu2022dpm,lu2022dpm++,zhang2022gddim, karras2022elucidating,zhao2023unipc} to improve the sampling process.
In particular, ODE samplers are proven effective with limited NFEs while SDE samplers are robust to prior mismatch~\cite{DBLP:conf/icml/LuZB0LZ22,nie2023blessing} and perform better in a sufficient number of NFEs~\cite{lu2022dpm++}. 

% These work do not benefits from \yz{?}

\textbf{Discrete DMs.}
Several DMs have been proposed to model discrete data with discrete states~\cite{sohl2015deep, hoogeboom2021argmax, austin2023structured}, depending on a probability transition matrix. It is nontrivial to leverage the features associated with continuous-state DMs, such as guidance and ODE fast sampling. Efforts have been made to define the score in the discrete state~\cite{lou2023discrete, meng2023concrete, campbell2022continuous, sun2023scorebased}; however, this remains a challenging endeavor. Other works~\cite{chen2023analog, dieleman2022continuous, li2022diffusionlm} have attempted to identify a continuous equivalent for discrete data and apply continuous DMs, but this may result in information loss during the transformation and greatly rely on the noise schedule~\cite{ye2023dinoiser}. \citet{mahabadi2023tess} defines a continuous-time diffusion process on continuous latent variables but is trained with cross-entropy loss rather than regression loss. Several studies~\cite{richemond2022categorical, lou2023reflected} have attempted to establish the diffusion process using SDEs on discrete data. Specifically, \citet{richemond2022categorical} introduced an SDE defined on the probability simplex, but it suffers from intractability in high-dimensional space. \citet{lou2023reflected} proposed a diffusion SDE with an additional boundary constraint, which also increases the complexity of discretization (e.g., requiring thresholding in SDE).

In comparison, this paper reveals that BFNs applied to discrete data solve a linear SDE and are trained using DSM, which aligns seamlessly with continuous DMs. Consequently, without changing the discrete data, BFNs are significantly simpler and more scalable and efficient than the related work, leveraging advancements in continuous DMs.

\section{Background}
\label{sec:background}

In this section, we present the elementary notations and background of DMs and BFNs.

\subsection{Elementary Notations}

% \yz{do we need a paragraph about our notations? like $\mathbf w$; bold means vector; $\mathbf x(t)$ and $\mathbf x_i$ }

We use lowercase letters (e.g., $t$) and boldface lowercase letters (e.g., $\rvx$) to denote scalars and vectors respectively. Variables indexed by uncountable indices are denoted in the form of functions, (e.g., $\beta(t)$ and $\vmu(t)$). Given finite indices (e.g. $\{t_i\}_{i=1}^M$), the corresponding variables are denoted with subscripts (e.g., $\vmu_i$).

\subsection{Score-based DMs}
Score-based DMs~\cite{kingma2021variational} characterize the data distribution
through a diffusion process $\{\rvx(t) \sim \mathcal{N}(\alpha(t) \rvx, \sigma^2(t)\mI)\}$ indexed by a continuous-time variable $t \in [0,T]$ according to an It\^o SDE as follows
\begin{align}
    \dif \rvx = {f}(t)\rvx\dif t + g(t)\dif \vw,\label{eq:bg-sde}
\end{align}
where $\vw$ is the standard Wiener process, and $f(t) = \frac{\dif \log \alpha(t)}{\dif t}$ and $g(t) = \frac{\dif \sigma^2(t)}{\dif t} - \frac{1}{2}\frac{\dif \log \alpha(t)}{\dif t}\sigma^2(t)$ are the drift and diffusion coefficients respectively. For instance, denoising diffusion probabilistic models~\cite{ho2020denoising} consider a process given by the following SDE:
\begin{align}
    \dif \rvx = -\frac{1}{2}\beta(t)\rvx \dif t + \sqrt{\beta(t)} \dif \vw,
\end{align}
where $0 < \beta(t) < 1$. Let $p_t(\rvx)$ denote the marginal density of $\rvx(t)$.  The generative process of score-based DMs is given by a reverse-time SDE~\cite{song2021scorebased, anderson1982reverse}
\begin{align}
    \dif \rvx = [f(t)\rvx - g(t)^2\nabla_\rvx\log p_t(\rvx)]\dif t + g(t)\dif \bar{\vw}\label{eq:bg-reverse-sde},
\end{align}
where $\bar \vw$ is the time-reversed Wiener process. Then the score is parameterized with a time-dependent score-based model $\hat\vs(\rvx, t)$ and trained with the following denoising score matching loss~\cite{vincent2011connection}
\begin{align}
    \mathcal{L}_{\mathrm{DSM}} = \E_{\rvx(t),\rvx(0)} [\|\hat\vs(\rvx(t), t) - \nabla_\rvx\log p_{0t}(\rvx(t) | \rvx(0))\|^2_2],\label{eq:bg-dsm}
\end{align}
where the conditional distribution $p_{0t}(\rvx | \rvx(0))$ is designed as a Gaussian kernel with a closed form score function $\nabla_\rvx\log p_{0t}(\rvx | \rvx(0))$. 
For fast sampling, \citet{song2021scorebased} introduce the corresponding \emph{probability flow ODE} of the reverse SDE in Eq.~\eqref{eq:bg-reverse-sde} as follows
\begin{align}
    \dif \rvx = \left[ {f}(t)\rvx - \frac{1}{2}g(t)^2\nabla_\rvx\log p_t(\rvx) \right] \dif t,\label{eq:bg-ode}
\end{align}
which produces the same data distribution as the corresponding SDE with infinitesimally small stepsize and enjoys a smaller discretization error with a large stepsize due to its deterministic nature~\cite{kloeden1992stochastic}. To solve the ODE in Eq.~\eqref{eq:bg-ode} efficiently, DPM-Solvers~\citep{lu2022dpm,lu2022dpm++} explicitly leverage the semi-linear property of Eq.~\eqref{eq:bg-ode} and further 
simplify it to an exponentially weighted integral of the neural network by applying change-of-variable. Consequently, the exact solution of ODE is given by
\begin{align}
    \rvx(t) = \frac{\alpha(t)}{\alpha(s)}\rvx(s) - \alpha(t)\int^{\lambda(t)}_{\lambda(s)}e^{-\lambda }\hat\ve_{\theta}(\hat\rvx(\lambda), \lambda)\dif \lambda,\label{eq:bg-dpm-solver-solution}
\end{align}
where  $\lambda(t) = \log(\alpha(t)/\sigma(t))$ is the half of the log signal-noise ratio. DPM-Solver solves Eq.~\eqref{eq:bg-dpm-solver-solution} numerically leading to a small discretization error. Taking DPM-Solver1 as an example, given time steps $\{t_i\}_{i=1}^{n}$ and initial value $\rvx_0$, a sequence $\{\rvx_i\}_{i=1}^{n}$ can be solved iteratively as follows:
\begin{align}
    \rvx_i \!=\! \frac{\alpha({t_i})}{\alpha({t_{i-1}})} \rvx_{i-1} \!-\! \sigma({t_i})(e^{h_i} - 1)\hat \vepsilon_\theta(\rvx_{i-1}, t_{i-1}) \!+\!\mathcal{O}(h_i^2),\label{eq:bg-dpm-solver-1}
\end{align}
where $h_i = \lambda(t_{i}) - \lambda(t_{i-1})$. Empirically, DPM-Solver achieves excellent results with a limited number of NFEs and is widely adopted.

\subsection{Bayesian Flow Networks}

Due to the space limit, we briefly present the motivation and formulation of BFNs~\cite{graves2023bayesian} here and please refer to the original paper for more details. Inspired by DMs, BFNs iteratively refine the parameters of a distribution set at different noise levels through Bayesian inference. This strategy enables BFNs to facilitate fast sampling and be differentiable on both continuous and discrete data.

For $D$-dimensional continuous data\footnote{We say $\rvx$ is a continuous data if its distribution has density w.r.t. the Lebesgue measure.} $\rvx \in \mathbb{R}^D$, a continuous-time BFN operates on parameters of a set of Gaussian distributions (of noisy data with different noise levels) with means $\{\vmu(t)\}_{t=0}^{1}$ and covariance matrices $\{\rho(t) \mI\}_{t=0}^{1}$. Equivalently, $\vmu(t)$ can also be regarded as a noisy version of $\rvx$ by injecting a Gaussian noise and follows the distribution
\begin{align}
    \flow(\vmu(t) | \rvx, \gamma(t)) = \mathcal{N}(\gamma(t)\rvx, \gamma(t)(1-\gamma(t))\mI),\label{eq:bg-cbfn}
\end{align}
where $\gamma(t) = 1 - \sigma_1^{2(1-t)}$ is a schedule function\footnote{For a clear alignment with DMs, we adopt a reverse time notation in this paper as originally used by~\citet{graves2023bayesian}. Specifically, the schedule $\gamma(t)$ in our paper is equivalent to $\gamma(1-t)$ in ~\citet{graves2023bayesian}. We retain the other notational conventions for ease of reading, which do not affect our derivations.} and $\sigma_1 \in (0, 1)$ is a hyperparameter. $\rho(t)$ is has a closed form as $\rho(t) = \frac{1}{1-\gamma(t)}$. Similar to DMs, a BFN on continuous data trains a neural network $\hat{\vepsilon}(\vmu(t), t)$ to predict the injected Gaussian noise $\vepsilon$ by minimizing the following loss:
\begin{align}
   \E_{\flow(\vmu(t) | \rvx, \gamma(t)), t\sim U(0, 1)} - \frac{\ln \sigma_1}{\sigma_1^{2t}} \| \vepsilon - \hat{\vepsilon}(\vmu(t), t) \|^2. \label{eq:bg-cbfn-train}
\end{align}
%Let $\rvu_i \sim \mathcal{N}(\boldsymbol{0}, \mI)$.
Given time steps $\{t_i\}_{i=0}^n$ and i.i.d.~noises $\{ \rvu_i\}_{i=0}^n \sim \mathcal N(0, \mI)$, the BFN sampler~\cite{graves2023bayesian} iterates as follows.
\begin{align}
    \vmu_i  &=  - \frac{\gamma(t_i) \!-\! \gamma(t_{i-1})}{\sqrt{\gamma(t_{i-1})(1 \!-\!\gamma(t_{i-1}))}} \hat{\vepsilon}(\vmu_{i-1}, t_{i-1}) 
    \!+\! \frac{\gamma(t_i)}{\gamma(t_{i-1})} \vmu_{i-1} \nonumber \\
    & \phantom{{}={}}\!+\! \sqrt{\frac{1\!-\!\gamma(t_i)}{1\!-\!\gamma(t_{i-1})}(\gamma(t_i) \!-\! \gamma(t_{i-1}))} \rvu_i.
    \label{eq:bg-cbfn-sample}
\end{align}
On $D$-dimensional discrete data $\rvx \in \{1,\cdots, K\}^D$, where $K$ is the number of classes, the BFN operates on parameters $\vtheta(t)$ of the multivariate categorical distributions of noisy data.  The distribution of $\vtheta$ is 
$$\flow(\vtheta(t) | \rvx, \beta(t)) = \E_{q(\rvz(t) | \rvx, \beta(t))} \delta(\vtheta(t) - \text{softmax}(\rvz(t))),$$
where $\delta(\cdot)$ is the Dirac distribution, $\rvz(t)$ is a set of latent variables with Gaussian marginal distributions as
\begin{align}
     q(\rvz(t) | \rvx, \beta(t))= \mathcal{N}(\beta(t) \vw_\rvx, K\beta(t)\mI),
      \label{eq:bg-dbfn}
\end{align}
and $\vw_\rvx := K\ve_\rvx - \boldsymbol{1}$, $\ve_\rvx :=\{\ve_{\rvx^{(1)}},\cdots,\ve_{\rvx^{(D)}}\}\in \mathbb{R}^{KD}$ where $\ve_j$ is the one-hot vector defined by $(\ve_{j})_k = \delta_{\rvx_{j}k}$ and $\boldsymbol{1}$ is a vector of length $KD$ filled with ones.  $\beta(t) = (1-t)^2 \beta_1$ is a schedule function with a hyperparameter $\beta_1>0$. A BFN on discrete data trains a neural network $\hat{\ve}(\vtheta(t), t)$ that predicts the data in a one-hot form given noisy inputs using the following regression loss
\begin{align}
\mathcal{L}^{\infty}(\rvx) \!=\! \E_{\flow(\vtheta|\rvx, t), t\sim U(0, 1)}K\beta_1t  \| \ve_\rvx - \hat{\ve}(\vtheta(t), t)\|^2. \label{eq:bg-dbfn-train}
\end{align}
Let $\{ \rvu_i\}_{i=0}^n \sim \mathcal N(0, \mI)$ be independent and use $\hat{\ve}_s(\rvz(t), t)$ as a shorthand for $\hat{\ve}(\text{softmax}(\rvz(t)), t)$.  The sampling rule of BFN~\cite{graves2023bayesian} can be written as follows
\begin{align}
    \ve_k &\sim \text{Cat}(\hat{\ve}_s(\rvz_{i-1}, t_{i-1})), 
    \label{eq:bg-dbfn-cat-sample}
    \\
    \rvz_i &=\rvz_{i-1} + \alpha_i (K \ve_k - 1) + \sqrt{K\alpha_i} \rvu_i,
    \label{eq:bg-dbfn-main-sample}
\end{align}
where $\alpha_i = \beta(t_i) - \beta(t_{i-1})$ and $\text{Cat}$ represents the one-hot categorical distribution.\footnote{Originally, \citet{graves2023bayesian} obtain samples through $\vtheta(t)$, while we present the equivalent form in terms of $\rvz(t)$ for convenience.}

Based on the formulation, BFNs have shown considerable promise in both image and language modeling.
Although inspired by DMs, and the exact relation between BFNs and DMs remains unclear. To this end, this paper unifies them through stochastic differential equations (SDEs) for understanding and accelerating BFNs on both continuous data (see Sec.~\ref{sec:CBCD}) and discrete data (see Sec.~\ref{sec:CBDD}).

\begin{figure}
    \centering
    \includegraphics[scale=0.8]{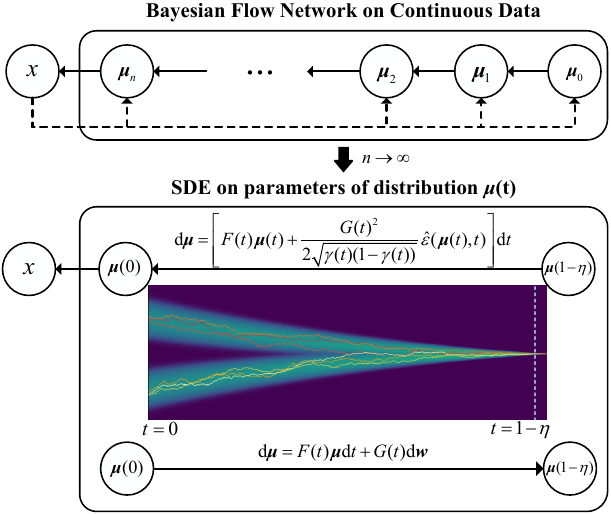}
    \vspace{-.15cm}
    \caption{\textbf{Illustration of BFN on continuous data and the corresponding SDEs.} The SDEs are defined w.r.t. $\vmu$ on time $[0, 1-\eta]$.}
    \label{fig:dag-continuous}
    \vspace{-.15cm}
\end{figure}

\section{Continuous-time BFN on Continuous Data}
\label{sec:CBCD}

% In sec.~\ref{sec:cbfn-as-sde}, we identify a linear SDE corresponding to the noise-adding processes in BFN on continuous data and derive the reverse-time SDE. Besides, the training objective of BFN is equivalent to DSM, as presented in Sec.~\ref{sec:cbfn-train}. The BFN sampler is proven as a first-order solver for the corresponding reverse time SDE in Sec.~\ref{sec:cbfn-sample}. Further, we derive fast samplers by discretizing the corresponding probability flow ODE based on the recipe in DMs in Sec.~\ref{sec:cbfn-ode}.

This section bridges BFNs on continuous data with DMs by establishing a linear SDE for noise modeling in BFN (Sec.~\ref{sec:cbfn-as-sde}), aligning training objectives with DSM (Sec.~\ref{sec:cbfn-train}), and validating the sampler as discretization of the reverse-time SDE (Sec.~\ref{sec:cbfn-sample}). Further,  fast samplers are developed based on the recipe in DMs in Sec.~\ref{sec:cbfn-ode}. 

\subsection{Formulating BFN on Continuous Data as SDEs}
\label{sec:cbfn-as-sde}

As illustrated in Fig.~\ref{fig:dag-continuous}, we establish that the (truncated) noise-adding process of the continuous-time BFN on continuous data in Eq.~\eqref{eq:bg-cbfn} uniquely solves a linear SDE, summarized as follows.

\begin{theorem}[Proof in Appendix \ref{proof:sde}]
\label{thm:cbfn-sde}
  Let $\eta > 0$ be an arbitrarily small constant. The BFN in Eq.~\eqref{eq:bg-cbfn} at time $[0, 1-\eta]$ is the unique solution of the following linear SDE:
\begin{align}
    \dif \vmu = F(t) \vmu \dif t + G(t) \dif \vw.
    \label{eq:proof_sde}
\end{align}
Here $\vw$ is a standard Wiener process and
\begin{align}
    F(t) &= \frac{\gamma^\prime(t)}{\gamma(t)} = 2 \frac{\sigma_1^{2(1-t)}}{1-\sigma_1^{2(1-t)}}\ln \sigma_1, \\
    G(t)^2 &= -\gamma^\prime(t) =  -2  {\sigma_1^{2(1-t)}}\ln \sigma_1,
%    F(t) &=  2 \frac{\sigma_1^{2(1-t)}}{1-\sigma_1^{2(1-t)}}\ln \sigma_1, \\
%    G(t)^2 &=  -2  {\sigma_1^{2(1-t)}}\ln \sigma_1,
\end{align}
where $\sigma_1 \in (0, 1)$ is the hyperparameter defined in Eq.~\eqref{eq:bg-cbfn}.
\end{theorem}
The time $t$ is truncated by $1 - \eta$ in Theorem~\ref{thm:cbfn-sde} for two reasons. On one hand, the reverse-time SDE derived later (see Eq.~\eqref{eq:cbfn-para-reverse-sde}) is ill-defined at $t = 1$ since the distribution of $\vmu$ collapses to a Dirac distribution whose score tends to infinity.
On the other hand, it is convenient to satisfy certain regularity conditions for the uniqueness of the solution in Theorem~\ref{thm:cbfn-sde}, as detailed in the proof. The exact distribution of $\vmu(1-\eta)$ is unknown. We approximate it by an isotropic Gaussian with zero mean and small variance, as $p(\vmu(t))$ tends towards a Dirac delta function when $t \to 1$. (see details in Sec.~\ref{sec:exp}). As $\eta$ is small (e.g., $10^{-3} \sim 10^{-5}$) in our implementation, the effect of truncation is negligible. 

Here a linear SDE also applies to the latent variable $\rvz$, a linear transformation of $\vmu$ in Eq.~\eqref{eq:bg-cbfn}.
The choice of $\vmu$ in Theorem~\ref{thm:cbfn-sde} aligns with the sampling process in BFN~\cite{graves2023bayesian}, facilitating a later analysis in Sec.~\ref{sec:cbfn-sample}. 

The finding in Theorem~\ref{thm:cbfn-sde} directly connects to DMs~\cite{song2021scorebased,kingma2021variational}, which are formulated as an SDE in Eq.~\eqref{eq:bg-sde} with a different noise schedule. 
We believe this may inspire new classes of BFNs and leave a systematic comparison of the schedules for future work.

Similar to Eq.~\eqref{eq:bg-reverse-sde}, the linear SDE in Eq.~\eqref{eq:proof_sde} has an associated reverse-time SDE~\cite{anderson1982reverse,song2021scorebased} in $[0, 1- \eta]$ for generative modeling:
\begin{align}
    \dif\vmu = [F(t)\vmu - G(t)^2\nabla_{\vmu} \log p_t(\vmu)] \dif t + G(t) \dif\bar{\vw}, \label{eq:cbfn-reverse-sde}
\end{align}
where $\nabla_{\vmu} \log p_t(\vmu)$ is the (time-conditional) score function to be estimated and $\bar{\vw}$ is the time-reversed Wiener process.

\subsection{Training as Parameterizing the Reverse-time SDE}
\label{sec:cbfn-train}

The continuous-time BFN on continuous data trains a neural network to optimize the mean square error in Eq.~\eqref{eq:bg-cbfn-train}, which directly aligns with the widely employed DSM loss in Eq.~\eqref{eq:bg-dsm}. In other words, BFN equivalently parameterizes the reverse-time SDE in Eq.~\eqref{eq:cbfn-reverse-sde} by estimating the time-conditional score function  as 
\begin{align} 
\hat{\vs}(\vmu(t), t) = -  \frac{1}{\sqrt{\gamma(t)(1-\gamma(t))}}\hat{\vepsilon}(\vmu(t), t), \label{eq:cbfn-rep}
\end{align}
where $\hat{\vs}(\vmu(t), t)$ and $\hat{\vepsilon}(\vmu(t), t)$ denote the estimate of the score function and the network trained by BFN, respectively, and $\gamma(t)$ follows Eq.~\eqref{eq:bg-cbfn}.

\subsection{Sampling as Discretizing the Reverse-time SDE}
\label{sec:cbfn-sample}

Plugging Eq.~\eqref{eq:cbfn-rep} into Eq.~\eqref{eq:cbfn-reverse-sde}, we get a parameterized reverse-time SDE in $[0, 1-\eta]$ for sampling as follows
\begin{align}
    \dif\vmu \!= \! \left[ \! F(t)\vmu(t) \!+\! \frac{G(t)^2\hat{\vepsilon}(\vmu(t), t)}{ \sqrt{\gamma(t)(1-\gamma(t))}} \!\right] \! \dif t \! + \! G(t) \dif\bar{\vw},
    \label{eq:cbfn-para-reverse-sde}
\end{align}
which is ill-defined at $t=1$ because $\lim_{t \rightarrow 1}\gamma(t) = 1$.
Interestingly, even without an explicit SDE formulation, the sampler proposed in the original BFN paper discretizes the reverse-time SDE, as characterized in the following Proposition~\ref{prop:convergence}. 
\begin{proposition}[Proof in Appendix~\ref{append:ancestral_sampling}]
\label{prop:convergence}
The BFN sampler in Eq.~\eqref{eq:bg-cbfn-sample} is a first-order discretization of an equivalent form of the parameterized reverse-time SDE in Eq.~\eqref{eq:cbfn-para-reverse-sde}.
\end{proposition}
% Moreover, for comparison, we propose an Euler discretization of the parameterized reverse-time SDE as follows
% \begin{align}
%     1, \label{eq:cbfn-euler-sde}
% \end{align}
% which surprisingly outperforms the original BFN sampler with a sufficient number of steps (see evidence in Sec.~\ref{sec:exp_fast_continous}).

\subsection{Probability Flow ODE and Faster Sampling}
\label{sec:cbfn-ode}

Establishing an explicit connection between BFNs and DMs through SDEs yields an immediate and significant benefit: the fast sampling recipe from DMs directly applies to BFN.
 
Formally, according to Eq.~\eqref{eq:bg-ode}, we obtain the following equivalent probability flow ODE of 
% Eq.~\eqref{eq:proof_sde}
% \begin{align}
%     \dif \vmu = \left[f(t)\vmu - \frac{1}{2}G(t)^2 \nabla_{\vmu} \log p_t(\vmu) \right] \dif t,
%     \label{eq:ode}
% \end{align}
% and that of 
the parameterized reverse-time SDE of Eq.~\eqref{eq:cbfn-para-reverse-sde}:
\begin{align}
    \dif \vmu \!= \! \left[F(t)\vmu(t)  \! +  \!\frac{G(t)^2}{ 2 \sqrt{\gamma(t)(1-\gamma(t))}}\hat{\vepsilon}(\vmu(t), t) \right] \dif t.\label{eq:cbfn-para-ode}
\end{align}

% Based on the recipe of DPM-Solvers, we propose \emph{BFN-Solvers}, a series of customized samplers for the probability flow ODE in BFN. By defining $\lambda(t) = \frac{1}{2}\log\frac{\gamma(t)}{1-\gamma(t)}$. Let time $t < s < 1-\eta$, The exact solution of the ODE is (see derivation in Appendix~\ref{append:cts_bfn_solver})
% \yz{A more straightforward way to derive the solver in the continuous case is to notice that the noise-injecting process \eqref{eq:bg-cbfn} is the same as the process $\mathcal N(\alpha_t, \sigma_t^2)$ in DPM with $\alpha_t = \gamma_t$ and $\sigma^2_t = \gamma_t(1 - \gamma_t)$ and we can directly copy DPM-Solver's algorithms, which is how the algorithms (for te continuous case) in the appendix are obtained. (also note that the current version is also acceptable, we can just leave it unchanged; so is the interp about the truncation below Thm 4.1)}
Further, we propose \emph{BFN-Solver}, a customized ODE solver for BFN in analogy to DPM-Solver in Eq.~\eqref{eq:bg-dpm-solver-1}. As detailed in Appendix~\ref{app:cts_bfn_solver}, we integrate all linear terms and apply a change of variable from $t$ to $\lambda(t) = \frac{1}{2}\log\frac{\gamma(t)}{1-\gamma(t)}$ to obtain a simplified exact solution of Eq.~\eqref{eq:cbfn-para-ode}
\begin{align}
    \vmu(t) \!=\! \frac{\gamma(t)}{\gamma(s)}\vmu(s) \!-\! \gamma(t) \int_{\lambda(s)}^{\lambda(t)} e^{-\lambda} \hat{\vepsilon}(\vmu(t_\lambda(\lambda)), t_\lambda(\lambda)) \dif \lambda, \label{eq:cts_ode_exact}
\end{align}
where $t_{\lambda}(\cdot)$ is the inverse function of $\lambda(t)$ for $0 \le t < s < 1-\eta$. Eq.~\eqref{eq:cts_ode_exact} differs from Eq.~\eqref{eq:bg-dpm-solver-solution} only in certain coefficients.
Given an initial value $\vmu_0$ and time steps $\{t_i\}_{i=0}^{n}$ from $t_0 = 1-\eta$ to $t_n = 0$, BFN-Solver1 is derived similarly to Eq.~\eqref{eq:bg-dpm-solver-1} and given by
\begin{align}
     \vmu_{i} =  & - \sqrt{\gamma(t_i)(1-\gamma(t_i))} (e^{h_i} - 1) \hat{\vepsilon}(\vmu_{i-1}, t_{i-1})  \nonumber \\ 
    & + \frac{\gamma(t_{i})}{\gamma(t_{i-1})}\vmu_{i-1},
\end{align}
where $h_i = \lambda(t_i) - \lambda(t_{i-1})$.
% and notably only the gradient of neural networks is approximated numerically to minimize the discretization error. 
\textit{BFN-Solver++} shares the same spirit with BFN-Solver and the difference is that BFN-Solver++ considers data prediction instead of noise prediction. We refer the readers to Appendix~\ref{app:cts_bfn_solver} for higher-order solvers of both ODE and SDE.\footnote{A more straightforward way to get BFN-Solver on continuous data is to treat BFN as a DM with a special noise schedule $\alpha_t = \gamma_t$ and $\sigma^2_t = \gamma_t(1 - \gamma_t)$. However, it is infeasible on discrete data. Therefore, we use a slightly complex yet coherent way to derive BFN-Solver throughout the paper.} 
% For simplicity, we present the first-order solver in the main text and refer the readers to Appendix~\ref{app:cts_bfn_solver} for higher-order solvers. Specifically, given an initial value $\vmu_0$ and time steps $\{t_i\}_{i=0}^{n}$ from $t_0 = 1-\eta$ to $t_n = 0$. The solvers use $n$ steps to iteratively compute a sequence $\{\vmu_{i}\}_{i=1}^{n}$ to approximate the true solution at time $\{t_i\}_{i=1}^n$. To compute $\vmu_{i-1 \rightarrow i} $, we need to approximate the exponentially weighted integral of $\hat{\vepsilon}$ from $\lambda(t_{i-1})$ to $\lambda(t_{i})$. Denote $h_i := \lambda(t_i) - \lambda(t_{i-1})$. Our first-order solver directly use $\hat{\vepsilon}(\mu_i, 1-t_{i-1})$ to approximate the $\hat{\vepsilon}$ in the integral, yields first-order solver iteration rule:
%To be consistent with BFN's sampler, $\vmu_n$ are sent to neural network, and the output is our final samples. 

Empirically, as presented in Sec.~\ref{sec:exp_fast_continous}, BFN-Solvers of different orders significantly outperform the original BFN sampler with a limited number of NFEs based on the same model.

\iffalse 
\subsection{Optimal Variance for Likelihood Evaluation}

\cx{
Recall that the loss function of discrete-time BFN on continuous data is to reduce KL divergence between the sender distribution and the receiver distribution in Eq.~\eqref{}.   
From the perspective of max likelihood estimation, we reveal that the optimal KL divergence has an analytic form. Based on that we could get an improved receiver distribution with an optimal variance and improved likelihood evaluation. 
}

\begin{proposition}[Proof in Appendix \ref{}]
Let receiver distribution be a Gaussian with a general function $\vxi_i(\vmu_{i-1})$ as mean and a data-dependent isotropic covariance $\lambda_i \mI$. Let sender distribution be $\mathcal{N}(\rvx, \alpha_i^{-1})$. The optimal KL divergence between the sender distribution and the receiver distribution is
\begin{align*}
    \frac{1}{2}\left [ D \log (\alpha_i \lambda^*_{i}) - D + D(\alpha_i\lambda^*_{i})^{-1} + \frac{1}{\lambda^*_{i}} \|\rvx-\vxi^*_i(\vmu_{i-1})\|^2
 \right]
\end{align*}
where
\begin{align*}
    &\lambda^*_{i} = \alpha_i^{-1} + \frac{1}{D} 
   \E_{\pdata(\rvx)} 
 \E_{\update(\vmu_{i-1}|\rvx)} \|\rvx- \vxi_i^*(\vmu_{i-1})\|^2, \\
 &\vxi_i^*(\vmu_{i-1}) = \E_{q(\rvx| \vmu_{i-1})}[\rvx].
\end{align*}
\end{proposition}

\cx{Empirically, we observe an improved likelihood results.}
\fi

\section{Continuous-time BFN on Discrete Data}
\label{sec:CBDD}

In a manner akin to Sec.~\ref{sec:CBCD}, this section unifies BFNs on discrete data and (continuous) DMs through SDEs and develops fast samplers for BFNs. 
However, this adaptation to discrete data is far from straightforward, as it involves SDEs operating on latent variables $\rvz$ — a significant departure from the original BFN formulation that marginalizes out these variables, rather than updating the distribution parameters $\vtheta$. Consequently, it is surprising that the training and sampling of BFN on discrete data still connect to the SDE formulation on $\rvz$.

\begin{figure}
    \centering
    \includegraphics[scale=0.8]{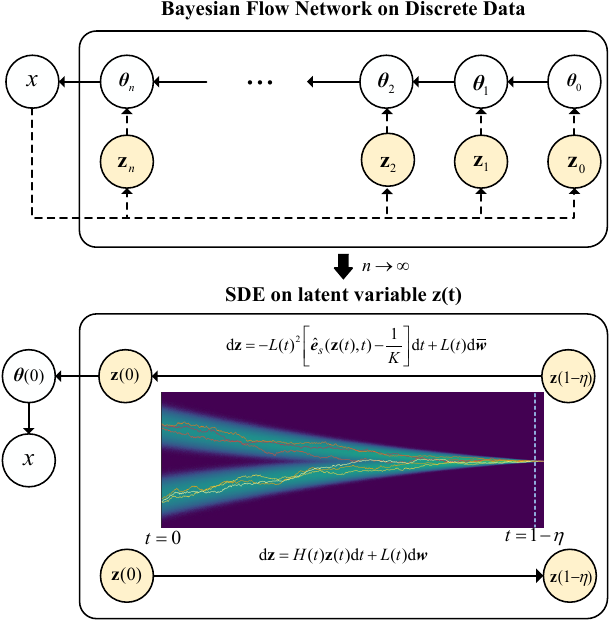}
    \vspace{-.15cm}
    \caption{\textbf{Illustration of BFN on discrete data and the corresponding SDEs.} The SDEs are defined w.r.t. the latent variables $\rvz$, which are marginalized in BFN, on time $[0, 1-\eta]$.}
    \label{fig:dag-discrete}
    \vspace{-.15cm}
\end{figure}

\subsection{Formulating BFN on Discrete Data as SDEs}
\label{sec:dbfn-as-sde}

Similar to Theorem~\ref{thm:cbfn-sde}, the truncated noise-adding process of the continuous-time BFN on discrete data in Eq.~\eqref{eq:bg-dbfn} uniquely solves a linear SDE, summarized as follows.
\begin{theorem}[Proof in Appendix~\ref{append:dbfn_sde}]
\label{thm:dbfn-sde}
    Let $\eta > 0$ be an arbitrarily small constant. The BFN in Eq.~\eqref{eq:bg-dbfn} with $t \in [0, 1-\eta]$ is the unique solution of the following linear SDE:
\begin{align}
    \dif \rvz = H(t) \rvz \dif t + L(t) \dif \vw. \label{eq:y_sde-main}
\end{align}
Here $\vw$ is a standard Wiener process and 
\begin{align}
    H(t) & = \frac{\beta^\prime(t)}{\beta(t)} = -\frac{2}{1-t}, \\
    L(t)^2 &= -K\beta^\prime(t) = 2K\beta_1(1-t), \label{eqn:discrete-diffusion-coeff}
\end{align}
where $K$ and $\beta_1$ are hyperparameters defined in Eq.~\eqref{eq:bg-dbfn}.
\end{theorem}
The rationale for truncation of $t$ and the way to deal with $\eta$ and $\rvz(1-\eta)$ is similar to the continuous data case, detailed in the proof and Sec.~\ref{sec:exp_setting}, respectively. 

Notably, Theorem~\ref{thm:dbfn-sde} characterizes the dynamics of $\rvz$ instead of $\vtheta$, 
as illustrated in Fig.~\ref{fig:dag-discrete}. Indeed, the dynamics of $\vtheta$ do not correspond to a linear SDE as $\vtheta$ is a nonlinear transformation of $\rvz$ as shown in Eq.~\eqref{eq:bg-dbfn}. It is implied that the original sampling process in Eq.~\eqref{eq:bg-dbfn-main-sample} \emph{does not directly} discretize the linear SDE, as detailed in Sec.~\ref{sec:dbfn-sample}.

The associated reverse-time SDE~\cite{song2021scorebased} for the linear SDE in Eq.~\eqref{eq:y_sde-main} in $[0, 1-\eta]$  is given by
\begin{align}
    \dif\rvz = [H(t)\rvz - L(t)^2\nabla_{\rvz} \log p_t(\rvz)] \dif t + L(t) \dif\bar{\vw}, \label{eq:dbfn-reverse-sde}
\end{align}
where $\nabla_{\rvz} \log p_t(\rvz)$ is the unknown score function, defined on $\rvz$ instead of $\vtheta$. 

\subsection{Training as Parameterizing the Reverse-time SDE}
\label{sec:dbfn-train}

It is nontrivial to see yet can be proven that the training objective of the continuous-time BFN on discrete data in Eq.~\eqref{eq:bg-dbfn-train} is a reparameterized form of DSM~\cite{vincent2011connection} w.r.t. $\rvz$, as summarized in the following Theorem~\ref{thm:discrete-dsm}.
\begin{theorem}[Proof in Appendix~\ref{app:discrete-dsm}] \label{thm:discrete-dsm}
    Minimizing the continuous-time loss of BFN on discrete data in Eq.~\eqref{eq:bg-dbfn-train} is equivalent to minimizing the DSM loss in Eq.~\eqref{eq:bg-dsm}. Besides, the corresponding estimate of the score function is given by
    \begin{align}
        \hat{\vs}(\rvz(t), t) = - \frac{\rvz(t)}{K\beta(t)} + \hat{\ve}_s(\rvz(t), t) - \frac{1}{K}, \label{eq:dbfn-dsm}
    \end{align}
    where $\hat{\ve}_s(\rvz(t), t) $ is the network trained by BFN.
\end{theorem}

Theorem~\ref{thm:dbfn-sde} and Theorem~\ref{thm:discrete-dsm} distinct BFNs from existing discrete DMs. Specifically, BFNs applied to discrete data solve a linear SDE and are trained using DSM, which aligns seamlessly with continuous DMs. Consequently, without changing the discrete data, BFNs are significantly simpler and more scalable and efficient than the related work, leveraging advancements in continuous DMs. We provide a comprehensive review and discussion in Sec.~\ref{sec:related}.

\subsection{Sampling as Discretizing the Reverse-time SDE}
\label{sec:dbfn-sample}

Plugging Eq.~\eqref{eq:dbfn-dsm} into Eq.~\eqref{eq:dbfn-reverse-sde}, we get a parameterized reverse-time SDE in $[0, 1-\eta]$ for sampling as follows
\begin{dmath}
    \dif\rvz =  - L(t)^2\left[\hat{\ve}_s(\rvz(t), t) - \frac{1}{K}\right]  \dif t + L(t) \dif\bar{\vw}. \label{eq:dbfn-para-reverse-sde}
\end{dmath}
The following Proposition~\ref{prop:convergence} suggests that the sampler proposed in the original BFN paper approximately discretizes the parameterized reverse-time SDE. 
\begin{proposition}[Proof in Appendix~\ref{app:dis_bfn_sampler}]\label{thm:dbfn-sample}
If the categorical sampling step in the BFN sampler on discrete data (i.e., Eq.~\eqref{eq:bg-dbfn-cat-sample}) is omitted, then it
is a first-order discretization of the parameterized reverse-time SDE in Eq.~\eqref{eq:dbfn-para-reverse-sde}.
\end{proposition}
The role of the categorical sampling step is still unclear in theory. However, experiments in Fig.~\ref{fig:spelling_bfn_vs_solversde1} (Sec.~\ref{sec:exp_fast_discrete}) reveal that removing the categorical sampling step leads to consistently better performance in fewer than $50$ NFEs, and almost the same results otherwise. We provide preliminary analyses of it in Appendix~\ref{app:ablation}.

% Moreover, we derive the following Euler sampler for the reverse-time SDE
% \begin{align}
%     1,
% \end{align}
% which performs better than the original BFN sampler with a large number of steps, as detailed in Sec.~\ref{sec:exp_fast_discrete}.

\subsection{Probability Flow ODE and Faster Sampling}
\label{sec:dbfn-fast}
 
Similar to the continuous case, the equivalent probability flow ODE of the parameterized reverse-time SDE on discrete data in 
Eq.~\eqref{eq:dbfn-para-reverse-sde} is 
\begin{dmath}
    \dif \rvz = \left\{-\frac{1}{1-t}\rvz(t) - \beta_1(1-t)[K\hat{\ve}_s(\rvz(t), t) -1]  \right\} \dif t. \label{eq:dbfn-para-ode}
\end{dmath}
For $0 \le t < s < 1-\eta$, its solution can be written as
\begin{dmath}
    \rvz(t) = \frac{1-t}{1-s}\rvz(s) + \beta_1(1-t)(t-s) - K\beta_1(1-t)\int_s^t \hat{\ve}_s(\rvz(\tau), \tau) \dif \tau. \label{eq:dis_ode_exact}
\end{dmath}
Again, we propose BFN-Solver on discrete data, and the first-order version is given by 
\begin{align}
     \rvz_i = & \beta_1(1-t_i)(t_i-t_{i-1})(1 - K\hat{\ve}_s(\rvz(t_{i-1}), t_{i-1})) \nonumber \\
     & +\frac{1-t_i}{1-t_{i-1}}\rvz_{i-1}.
\end{align}
Notably, we map the latent $\rvz_M$ to the distribution parameter $\vtheta_M = \text{softmax}(\rvz_M)$ at the last step to obtain the final samples. We refer the readers to Appendix~\ref{app:dis_bfn_solver} for higher-order solvers of both ODE and SDE. As presented in Sec.~\ref{sec:exp_fast_discrete}, the conclusion on the improvement of BFN-Solvers over the original BFN sampler remains the same on discrete data.

\section{Experiments}
\label{sec:exp}

We present the experimental setups in Sec.~\ref{sec:exp_setting}. We validate the proposed BFN-Solvers on continuous and discrete data, in Sec.~\ref{sec:exp_fast_continous} and Sec.~\ref{sec:exp_fast_discrete} respectively.

\subsection{Experimental Settings}
\label{sec:exp_setting}
\textbf{Model. }We employed the pre-trained models provided by the BFN~\citep{graves2023bayesian} in all experiments for fairness.

\textbf{Datasets. }For continuous data, the model is trained on the CIFAR-10~\citep{krizhevsky2009learning} dataset which contain 50K training images. For discrete data, the model is trained on the text8~\citep{mahoney2011large} dataset which contains 90M consecutive characters, each character is a lower Latin letter ‘a’-‘z’ or the ‘$\_$’ whitespace token, giving a class number of 27. Each sample is a sequence of 256 characters.

\textbf{Metrics. }For continuous data, we adopt the widely used FID~\citep{heusel2017gans} as the sample quality metric. We compute the FID metric on 10K generated samples for efficiency. For discrete data, there is no widely adopted sample-based metric comparable to FID in image modeling. Given our reliance on a simple character-level text dataset, we found that spelling accuracy (SA) is a straightforward yet effective metric for measuring the quality of text generation. Specifically, SA is defined as the ratio of correctly spelled words to the total words in the entire generated sequence, which is segmented by spaces. In each experiment, we collect 1,000 generated samples to calculate the metric. Additionally. we conducted a user study for text generation quality evaluation. For the user study, there are $100$ questions for each one vs.~one comparison (e.g., BFN vs. BFN-Solver1). In each question, participants were presented with two sentences randomly generated from two methods. Participants were instructed to choose a sentence of higher quality, which is known as the \emph{two-alternative forced choice} methodology~\cite{kawar2023imagic, bar2022text2live, park2020swapping}. Please see Appendix~\ref{app:experimental_details} for more experimental details.

\textbf{Truncation. } $\eta$ is a manually tuned hyperparameter specified in each experiment. For both $p_{1-\eta}(\vmu)$ and $p_{1-\eta}(\rvz)$, we found an isotropic Gaussian with zero mean and a calculated variance works well. We provide preliminary analyses of the variance in Appendix~\ref{sec:optimal-init}. 

\subsection{Fast Sampling on Continuous Data}
\label{sec:exp_fast_continous}
We compare our proposed fast sampling methods with the original BFN continuous sampler in this section. 

As illustrated in Fig.~\ref{fig:fid_nfe}, with the NFE less than $100$, BFN-Solver++1, BFN-Solver++2, and SDE-BFN-Solver++2 significantly outperform the BFN baseline. Moreover, BFN-Solver++2 achieves better results compared to BFN-Solver++1. When the NFE is higher (e.g., more than $500$), our observations reveal that SDE-based samplers exhibit slightly better performance over ODE-based samplers, which aligns with the diffusion model~\citep{song2020denoising, karras2022elucidating, nie2023blessing}. Please see Appendix~\ref{app:results_continuous} and Appendix~\ref{app:random_samples} for more quantitative results and randomly generated images, respectively.

We slightly tune the hyperparameter $\eta$ for our methods on different NFEs to get the best results, as detailed in Appendix~\ref{app:analysis_eta_continous}.

\begin{figure}[t!]
    \centering
    \includegraphics[width=\linewidth]{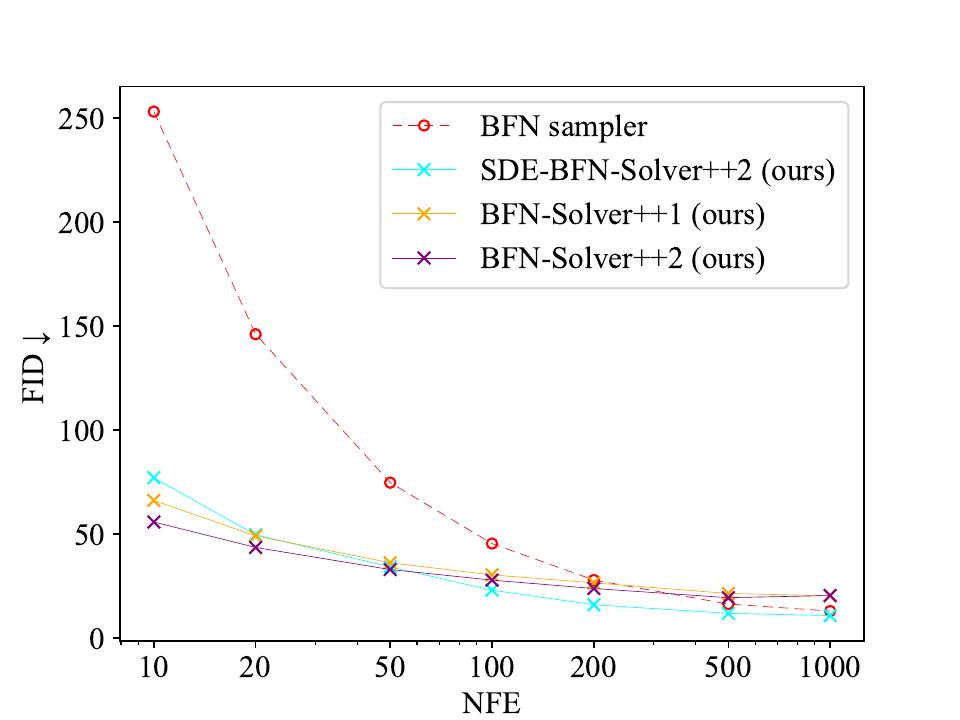}
    \vspace{-.15cm}
    \caption{\textbf{Fast sampling results on the  continuous CIFAR-10 dataset.} Sampling quality is measured by FID $\downarrow$, varying the NFE.}
    \label{fig:fid_nfe}
    \vspace{-.15cm}
\end{figure}

\subsection{Fast Sampling on Discrete Data}
\label{sec:exp_fast_discrete}
We compare our proposed fast sampling methods with the origin BFN discrete sampler and other representative baselines in this section. 
% Our methods include the BFN-Solver1, BFN-Solver2 and SDE-BFN-Solver2. 

As illustrated in Fig.~\ref{fig:spelling_nfe}, with the NFE less than $30$, BFN-Solver1, BFN-Solver2, and SDE-BFN-Solver2 significantly outperform the BFN baseline. Moreover, BFN-Solver2 and SDE-BFN-Solver2 achieve better results compared to BFN-Solver1, agreeing with the continuous case. We provide a preliminary user study in Fig.~\ref{fig:user_study} with 10 NFEs and the results align with Fig.~\ref{fig:spelling_nfe}.
When the NFE is higher (e.g., more than $500$), we observe that SDE-based samplers exhibit slightly better performance than ODE-based samplers, which aligns with existing results in DMs. Please see Appendix~\ref{app:results_discrete} and Appendix~\ref{app:random_samples} for more quantitative results and randomly generated texts. 

We compare the efficiency of BFN-Solver2 with representative baselines including D3PM~\cite{austin2023structured} and OAARDM~\cite{HoogeboomGBPBS22}. As summarized in Table~\ref{tab:comp}, BFN-Solver2 outperforms all baselines with both $20$ and $50$ steps indirectly. Notably, evaluating the exact likelihood by the proposed samplers requires a nontrivial extension~\citep{song2021scorebased}. We use the original BFN sampler, evaluated under both metrics, as a bridge for an indirect comparison. For instance, according to Table~\ref{tab:comp}, BFN-Solver2 ($20$ steps, SA $0.84$) is better than the original BFN sampler ($25$ steps, SA $0.80$, BPC $1.52$) and better than OA-ARDM ($20$ steps, BPC $1.52$).
\begin{table}[t!]
\caption{\textbf{Efficiency evaluation results of different methods on discrete text8 dataset.} model performance is measured by bits-per-character~(BPC) $\downarrow$ and SA$\uparrow$, varying the number of evaluations. To get the strongest baseline, we use the D3PM absorbing model, which has a better BPC than D3PM uniform and NN models in the original paper~\cite{austin2023structured}.}
\label{tab:comp}
\vskip 0.15in
\begin{center}
\begin{small}
\begin{sc}
\begin{tabular}{lccr}
\toprule
Methods  & NFE   & BPC $\downarrow$   & SA $\uparrow$  \\
\midrule
BFN-Solver2 (Ours)    & 50 & -    & \textbf{0.85} \\
BFN                  & 50 & 1.47 & 0.84          \\
D3PM (Best)          & 50 & 1.52 & -             \\
BFN                  & 25 & 1.52 & 0.80          \\
BFN-Solver2 (Ours)    & 20 & -    & \textbf{0.84} \\
BFN                  & 20 & 1.56 & 0.79          \\
BFN-Solver2 (Ours)    & 20 & 1.56 & -             \\
OA-ARDM              & 20 & 1.52 & -             \\
\bottomrule
\end{tabular}
\end{sc}
\end{small}
\end{center}
\vskip -0.1in
\end{table}

% \yz{@CX: I added a section in Appendix.~\ref{sec:optimal-init} to explain the choice of the initialization (in the continuous case, the discrete case is similar), if you think this explanation is reasonable, you can keep it, and mention it somewhere; otherwise you can directly remove it (since the more we say, the more the reviewers can ask)
% @KAIWEN: check the proof and discussion in that section.}
% Similarly to the continuous case, we found that for a small $\eta$, an isotropic Gaussian $\mathcal{N}(\boldsymbol{0}, K\beta(1-\eta)\mI)$ is good to approximate the true distribution of $\vtheta(1-\eta)$.
We find that the hyperparameter $\eta=0.001$ is sufficient for all NFEs for BFN-Solvers to get excellent results. We refer the readers to Appendix~\ref{app:analysis_eta_discrete} for more details.

Finally, we perform an ablation of the original BFN solver in Fig.~\ref{fig:spelling_bfn_vs_solversde1} and find that an exact solver that just removes the categorical sampling step from the BFN sampler works better, conforming to our theory.

\begin{figure}[t!]
    \centering
    \includegraphics[width=\linewidth]{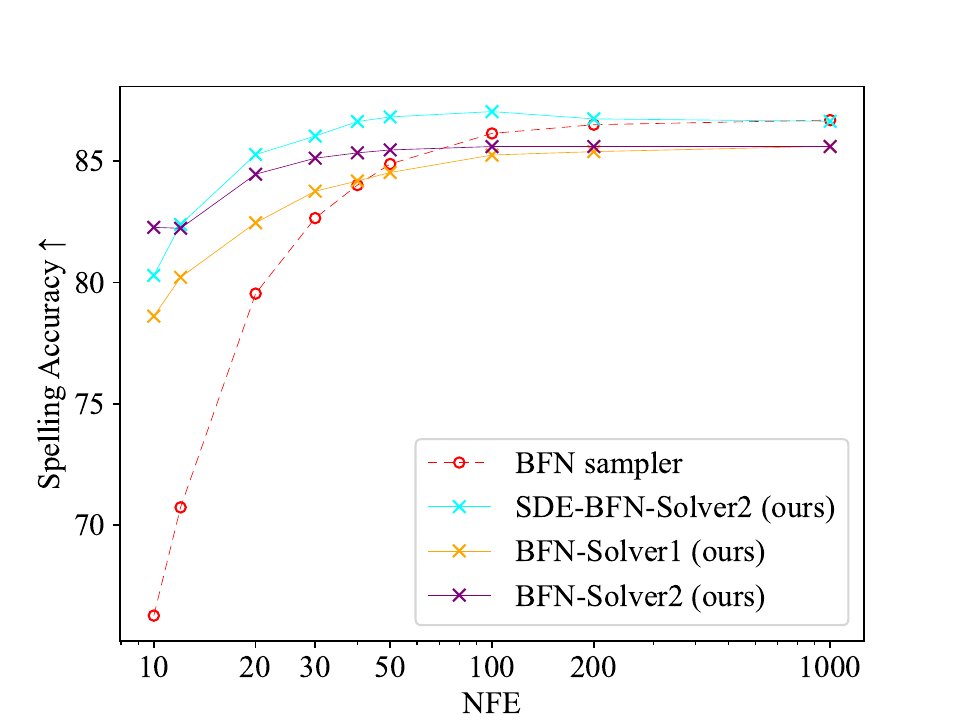}
    \vspace{-.15cm}
    \caption{\textbf{Fast sampling results on the discrete text8 dataset.} Sampling quality is measured by SA $\uparrow$, varying the NFE.}
    \label{fig:spelling_nfe}
    \vspace{-.15cm}
\end{figure}

\begin{figure}[t!]
    \centering
    \includegraphics[width=\linewidth]{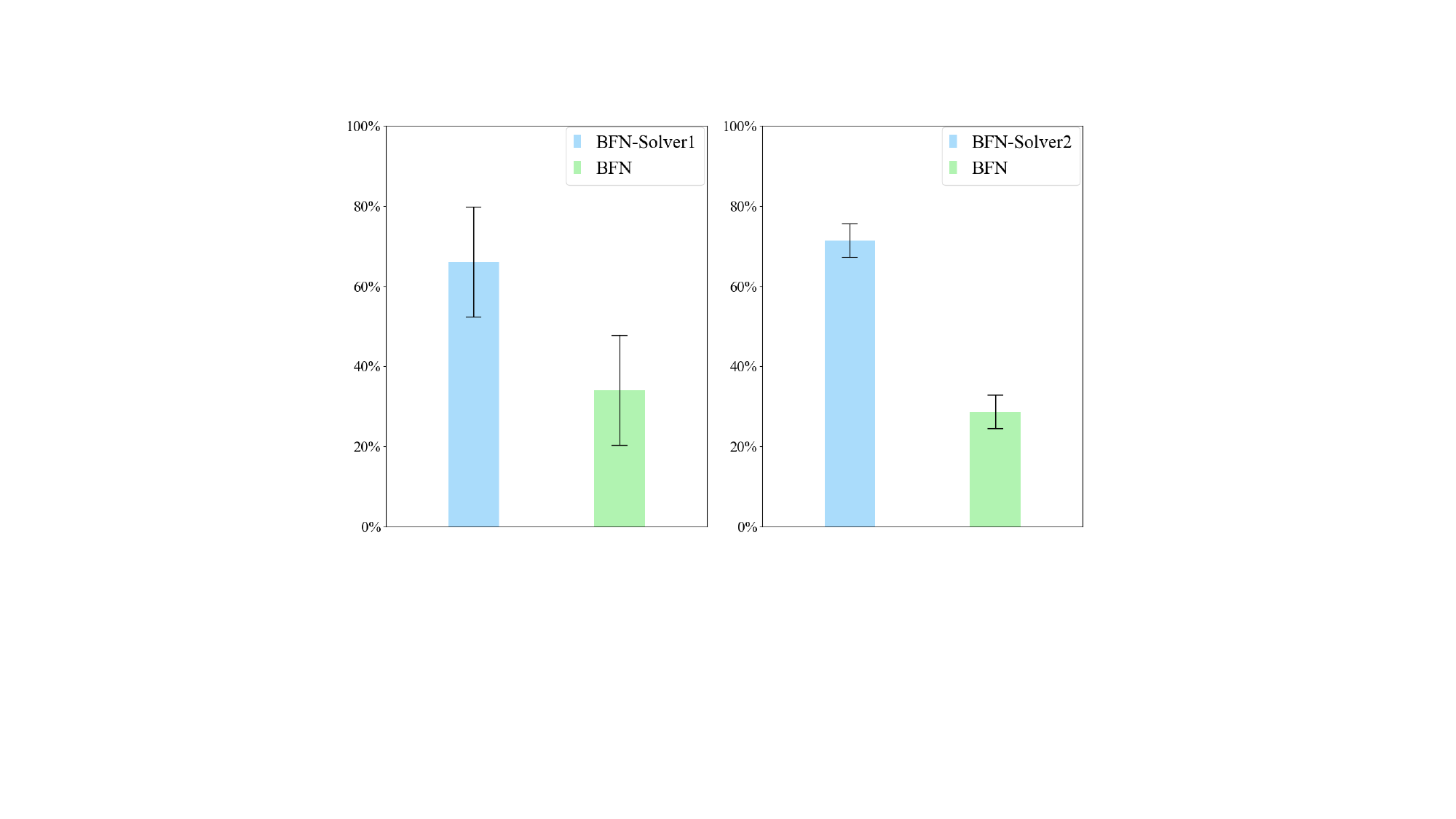}
    \vspace{-.15cm}
    \caption{\textbf{User study results on the discrete text8 dataset with 10 NFE.} We present the preference rates (with 95\% confidence intervals) of BFN-Solver1 and BFN-Solver2 over BFN baseline.}
    \label{fig:user_study}
    \vspace{-.15cm}
\end{figure}

\begin{figure}[t!]
    \centering
    \includegraphics[width=\linewidth]{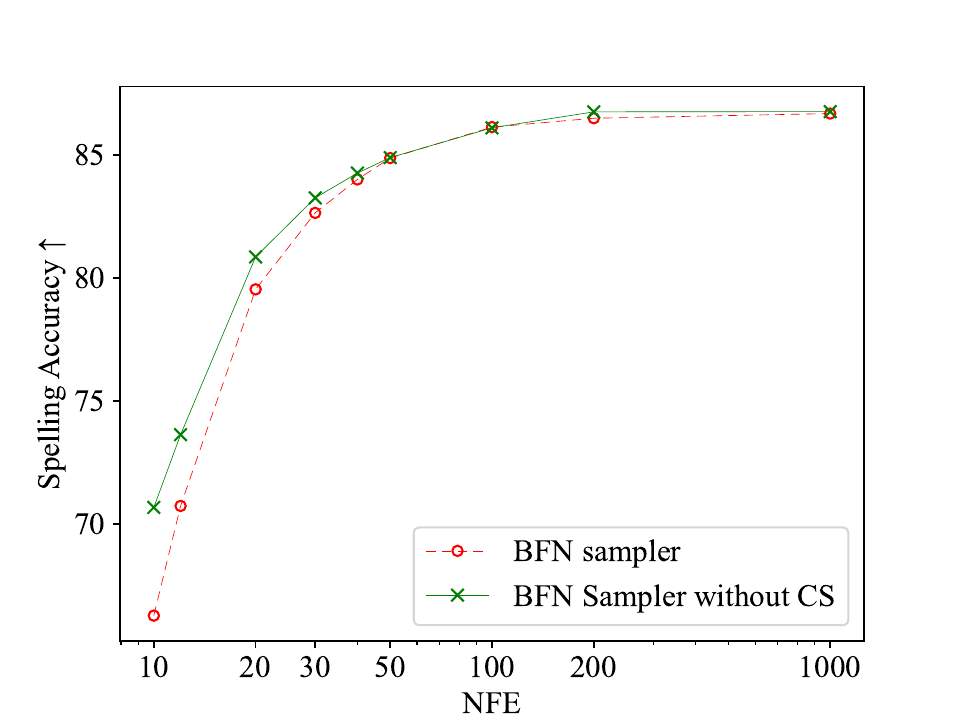}
    \vspace{-.15cm}
    \caption{\textbf{Ablation of the categorical sampling (CS) step in the BFN sampler on the text8 dataset.} Sampling quality is measured by SA $\uparrow$, varying the NFE.}
    \vspace{-.15cm}
    \label{fig:spelling_bfn_vs_solversde1}
\end{figure}

\section{Conclusion}
\label{sec:conclusion}

We unify BFNs and DMs by identifying the linear SDEs pertinent to the noise-addition processes in BFNs, illustrating that BFN's regression losses correspond with denoise score matching, and validating the sampler in BFN as an effective first-order solver for the related reverse-time SDE. Motivated by these insights, we implement fast sampling techniques from DMs in BFNs, yielding promising results.

Building upon the established results of DMs,
this paper establishes a principled and systematic approach to the analysis and enhancement of BFNs and future work includes the development of predictor-corrector samplers~\cite{song2020denoising,zhao2023unipc}, improved methods for likelihood evaluation~\cite{bao2022analytic,bao2022estimating}, and novel training strategies to refine~\cite{karras2022elucidating} and scale BFNs~\cite{rombach2022high}.

Limitations of the paper include the scale of the datasets and evaluation metrics. In our experiment, for a fair comparison, we leverage the pre-trained models of BFNs, which are all trained on small datasets. Further, the samplers cannot be directly used in likelihood evaluation and we mainly employ the FID and spelling accuracy as surrogates for the sample quality, potentially introducing bias. Hopefully, these limitations can be solved by scaling up BFNs to common benchmarks, as mentioned in future work.

\section*{Impact Statement}

This paper presents work whose goal is to advance the field of machine learning. There are many potential societal consequences of our work, none of which we feel must be specifically highlighted here.

\section*{Acknowledgements}

This work was supported by NSF of China (No. 62076145), Major Innovation \& Planning Interdisciplinary Platform for the ``Double-First Class" Initiative, Renmin University of China, the Fundamental Research Funds for the Central Universities, the Research Funds of Renmin University of China (No. 22XNKJ13), and the Ant Group Research Fund. C. Li was also sponsored by Beijing Nova Program (No. 20220484044). The work was partially done at the Engineering Research Center of Next-Generation IntelligentSearch and Recommendation, Ministry of Education.

\bibliography{icml2024}
\bibliographystyle{icml2024}

%%%%%%%%%%%%%%%%%%%%%%%%%%%%%%%%%%%%%%%%%%%%%%%%%%%%%%%%%%%%%%%%%%%%%%%%%%%%%%%
%%%%%%%%%%%%%%%%%%%%%%%%%%%%%%%%%%%%%%%%%%%%%%%%%%%%%%%%%%%%%%%%%%%%%%%%%%%%%%%
% APPENDIX
%%%%%%%%%%%%%%%%%%%%%%%%%%%%%%%%%%%%%%%%%%%%%%%%%%%%%%%%%%%%%%%%%%%%%%%%%%%%%%%
%%%%%%%%%%%%%%%%%%%%%%%%%%%%%%%%%%%%%%%%%%%%%%%%%%%%%%%%%%%%%%%%%%%%%%%%%%%%%%%
\newpage
\appendix

\onecolumn

\section{Derivation for Continuous-time BFN on Continuous Data}
\subsection{Proof of Theorem~\ref{thm:cbfn-sde}}
\label{proof:sde}
\begin{proof}
    In this proof, we will show that the marginal distribution $\vmu(t)$ \emph{conditioned on $\rvx$} of Eq.~\eqref{eq:proof_sde} at time $t$ is the BFN in Eq.~\eqref{eq:bg-cbfn} 
    by the ``variation of constants'' formula~\citep{atkinson2009numerical} and the It\^o's formula~\citep[Theorem 4.2.1]{oksendal2003stochastic}.

    To find the marginal distribution of $\vmu(t)$, we introduce a new process $\tilde \vmu(t) := e^{-\int_0^t F(\tau)\dif \tau} \vmu(t)$.
    Using It\^o's formula, we see that $\tilde \vmu$ follows the equation below.
    \begin{align*}
        \dif \tilde \vmu 
        &= e^{-\int_0^t F(\tau)\dif \tau} \dif \vmu(t) + \left ( \frac{\dif}{\dif t} e^{-\int_0^t F(\tau)\dif \tau} \right ) \vmu(t) \dif t \\
        &= e^{-\int_0^t F(\tau)\dif \tau} (F(t) \vmu(t) \dif t + G(t) \dif \vw) - F(t) e^{-\int_0^t F(\tau)\dif \tau} \vmu(t) \dif t \\
        &= e^{-\int_0^t F(\tau)\dif \tau} G(t) \dif \vw.
    \end{align*}
    Writing the above equation in the integral form we find that 
    \begin{equation}
        \tilde \vmu(t)
       = \tilde \vmu(0) 
        +  \int_0^t e^{-\int_0^s F(\tau)\dif \tau} G(s) \dif \vw(s).
        \label{eqn:change-of-variable-tilde-mu}
    \end{equation}
    It is known that from It\^o's isometry~\citep[Corollary 3.1.7]{oksendal2003stochastic} the above It\^o's integral is a Gaussian distribution with variance 
        $\int_0^t e^{-\int_0^s 2F(\tau)\dif \tau} G(s)^2 \dif s$.

    Note that $F(t) = \frac{\dif}{\dif t} \ln \gamma(t)$ and $G(t)^2 = -\frac{\dif}{\dif t} \gamma(t)$, then Eq.~\eqref{eqn:change-of-variable-tilde-mu} becomes
    \begin{equation*}
        \frac{\gamma(0)}{\gamma(t)} \vmu(t) \sim \vmu(0) + \mathcal N\left( 0, \gamma(0)^2\left( \frac{1}{\gamma(t)} - \frac{1}{\gamma(0)} \right) \mI \right),
    \end{equation*}
    which shows that the distribution of $\vmu(t)$ conditioned on $\vmu(0)$ is
    \begin{equation}
        \label{eqn:continuous-cond-mu0}
        \vmu(t) \mid \vmu(0) \sim \mathcal N\left( \frac{\gamma(t)}{\gamma(0)} \vmu(0), \left (\gamma(t) - \frac{\gamma(t)^2}{\gamma(0)} \right ) \mI\right).
    \end{equation}
    Recall that from Eq.~\eqref{eq:bg-cbfn} with $t = 0$, the initial distribution of $\vmu(0)$ given the data $\rvx$ is 
    \begin{equation}
        \label{eqn:continuous-cond-x}
        \vmu(0) \mid \rvx \sim \mathcal N(\gamma(0) \rvx, \gamma(0)(1 - \gamma(0)) \mI).
    \end{equation}
    In view of the above two equations, we see that
    \begin{equation}
        \vmu(t) \mid \rvx \sim \mathcal N(\gamma(t) \rvx, \gamma(t)(1 - \gamma(t)) \mI),
    \end{equation}
    which is the BFN in Eq.~\eqref{eq:bg-cbfn} for any $t \in [0, 1 - \eta]$, and thus completes the proof.
\end{proof}

    Finally, we note that the coefficient $F(t)$ tends to infinity as $t \to 1$, which may violate the regularity assumptions used in It\^o's formula, and in the existence and the uniqueness of the solution of Eq.~\eqref{eq:proof_sde}. 
    Nevertheless, as noted in the paragraph after Theorem~\ref{thm:cbfn-sde}, when we restrict on the time interval $t \in [0, 1 - \eta]$ for any fixed $\eta > 0$, the coefficients of Eq.~\eqref{eq:bg-cbfn} are well-behaved.

\subsection{Proof of Proposition \ref{prop:convergence}}
\label{append:ancestral_sampling}

In this section, we provide the proof of Proposition~\ref{prop:convergence}. As we will see, the BFN sampler is equivalent to solving a reparametrized equation using a variant of the first-order DPM-Solver++~\cite{lu2022dpm++}.

\begin{proof}
We consider the following reparameterization of $\hat \vepsilon$:
\begin{equation}
    \hat \rvx(\vmu, t) := \frac{\vmu}{\gamma(t)} - \sqrt{\frac{1 - \gamma(t)}{\gamma(t)}} \hat \vepsilon(\vmu, t) 
    \quad \Leftrightarrow \quad
    \hat \vepsilon(\vmu,t ) = \sqrt{\frac{\gamma(t)}{1 - \gamma(t)}}\left [\frac{\vmu}{\gamma(t)} - \hat \rvx(\vmu, t) \right ]
    ,
\end{equation}
under which the reverse SDE in Eq.~\eqref{eq:cbfn-para-reverse-sde} is
\begin{align}
    \dif \vmu 
    &= \left(  
        F(t)\vmu + \frac{G(t)^2}{\sqrt{\gamma(t)(1 - \gamma(t))}} \hat \vepsilon(\vmu(t), t) 
     \right) \dif t + G(t) \dif \bar \vw \nonumber \\
    &= \left(
        \frac{\gamma^\prime(t)}{\gamma(t)} \vmu - \frac{\gamma^\prime(t)}{\sqrt{\gamma(t)(1 - \gamma(t))}} \hat \vepsilon(\vmu(t), t) 
     \right) \dif t + \sqrt{-\gamma^\prime(t)} \dif \bar \vw \nonumber \\
    &= \left(
        \frac{\gamma^\prime(t)}{\gamma(t)} \vmu - \frac{\gamma^\prime(t)}{1 - \gamma(t)} \left ( \frac{\vmu}{\gamma(t)} - \hat \rvx(\vmu(t), t) \right )  
     \right) \dif t + \sqrt{-\gamma^\prime(t)} \dif \bar \vw \nonumber \\
    &= \left(
        -\frac{\gamma^\prime(t)}{1 - \gamma(t)}\vmu + \frac{\gamma^\prime(t)}{1 - \gamma(t)} \hat \rvx(\vmu(t), t) 
     \right) \dif t + \sqrt{-\gamma^\prime(t)} \dif \bar \vw.
\end{align}
In the integral form similar to Eq.~\eqref{eqn:change-of-variable-tilde-mu}, for $0 \leq t < s < 1$, the above equation is 
\begin{align}
    \vmu(t) = \frac{1 - \gamma(t)}{1 - \gamma(s)} \vmu(s) + (1 - \gamma(t)) \int_s^t \frac{\gamma^\prime(\tau)}{(1 - \gamma(\tau))^2} \hat \rvx(\vmu(\tau), \tau) \dif \tau
    + \int_s^t \frac{1 - \gamma(t)}{1 - \gamma(\tau)} \sqrt{-\gamma^\prime(\tau)} \dif \bar\vw(\tau).
\end{align}
Approximating $\hat \rvx(\vmu(\tau), \tau)$ using $\hat \rvx(\vmu(s), s)$, and compute the It\^o's integral exactly, we can obtain the following first-order discretization
\begin{align*}
    \vmu(t) 
    &\approx \frac{1 - \gamma(t)}{1 - \gamma(s)} \vmu(s) + \left( \frac{\gamma(t) - \gamma(s)}{1 - \gamma(s)} \right) \hat \rvx(\vmu(s), s)
    + \sqrt{\frac{1 - \gamma(t)}{1 - \gamma(s)}(\gamma(t) - \gamma(s))} \rvu_{s} \\
    &= \frac{\gamma(t)}{\gamma(s)} \vmu(s) 
    - \frac{\gamma(t) - \gamma(s)}{\sqrt{\gamma(s)(1 - \gamma(s))}}  \hat \vepsilon(\vmu(s), s)
    + \sqrt{\frac{1 - \gamma(t)}{1 - \gamma(s)}(\gamma(t) - \gamma(s))} \rvu_{s},
\end{align*}
where $\rvu_{s} = \mathcal{N}(\boldsymbol{0}, \mI)$, which matches the BFN's sampling rule in Eq.~\eqref{eq:bg-cbfn-sample}, and thus completes the proof.
\end{proof}

\subsection{Derivation of BFN-Solvers on Continuous Data}
\label{app:cts_bfn_solver}
We follow the principles in DPM-Solvers~\citep{lu2022dpm,lu2022dpm++} to derive samplers of the ODE in Eq.~\eqref{eq:cbfn-para-ode}.
Let times $0 \leq t < s \leq 1-\eta$ and $\lambda(t) = \frac{1}{2}\log\frac{\gamma(t)}{1-\gamma(t)}$. The exact solution of this ODE can be formulated by ``variation of constants" formula~\cite{atkinson2009numerical}:
\begin{align}
    \vmu(t) &= e^{\int_s^t F(\tau) \dif \tau} \vmu(s) + \int_s^t\left( e^{\int_\tau^t F(r)\dif r} \frac{G(\tau)^2}{ 2 \sqrt{\gamma(\tau)(1-\gamma(\tau))}}\hat{\vepsilon}(\vmu(\tau), \tau) \right) \dif \tau.\label{eq:cts_exact_pro}
\end{align}
Note that $F(t) = \frac{\dif}{\dif t} \ln \gamma(t)$ and $\frac{\dif \lambda(t)}{\dif t} = -\frac{G(t)^2}{\gamma(t)(1-\gamma(t))}$, then the above equation becomes
\begin{align}
    \vmu(t) &= \frac{\gamma(t)}{\gamma(s)}\vmu(s) - \gamma(t) \int_{s}^{t} \frac{\dif \lambda(\tau)}{\dif \tau}e^{-\lambda(\tau)} \hat{\vepsilon}(\vmu(\tau), \tau) \dif \tau \\
    &= \frac{\gamma(t)}{\gamma(s)}\vmu(s) - \gamma(t) \int_{\lambda(s)}^{\lambda(t)} e^{-\lambda} \hat{\vepsilon}(\vmu(t_\lambda(\lambda)), t_\lambda(\lambda)) \dif \lambda,
    \label{eqn:change-of-var-ode-cont}
\end{align}
where the last equation uses the fact that $\lambda(\tau)$ is strict monotone and $t_{\lambda}(\cdot)$ is the inverse function of $\lambda(t)$. 
% as $\lambda(t)$ is a strictly decreasing function.
%We can verify that $\lambda(t)$ is a strictly decreasing function w.r.t $t$ by computing the derivative $\frac{\dif \lambda(t)}{\dif t} = -\frac{G(t)^2}{\gamma(t)(1-\gamma(t))} < 0$. 
%Then, we can rewrite Eq.~\eqref{eq:cts_exact_pro} as
%\begin{align}
%    \vmu(t) = \frac{\gamma(t)}{\gamma(s)}\vmu(s) - \gamma(t) \int_{\lambda(s)}^{\lambda(t)} e^{-\lambda} \hat{\vepsilon}(\vmu(t_\lambda(\lambda)), t_\lambda(\lambda)) \dif \lambda.
%\end{align}
%where $t_{\lambda}(\cdot)$ is the inverse function of $\lambda(t)$. % as $\lambda(t)$ is a strictly decreasing function.

Let $\hat \ve^{(m)}(\lambda) := \hat \ve^{(m)}(\vmu(t_\lambda(\lambda)), t_\lambda(\lambda)) := \frac{\dif^m \hat \ve(\vmu(t_\lambda(\lambda)), t_\lambda(\lambda))}{\dif \lambda^m}$ be the $m$-th order derivative of the map $\lambda \mapsto \hat \ve(\vmu(t_\lambda(\lambda)), t_\lambda(\lambda))$, then we can approximate $\hat \ve(\vmu(t_\lambda(\lambda)), t_\lambda(\lambda))$ by the Taylor's expansion at $\lambda = \lambda(s)$ for any $k \geq 0$:
\begin{align}
    \hat{\vepsilon}(\vmu(t_\lambda(\lambda)), t_\lambda(\lambda))  =
    \sum_{m=0}^{k-1} \frac{(\lambda - \lambda(s))^m}{m!}\hat\ve^{(m)}(\lambda(s)) + O((\lambda - \lambda(s))^k).
\end{align}
Subsituting the above equation into Eq.~\eqref{eqn:change-of-var-ode-cont} yields 
\begin{align}
    \vmu(t)  &= \frac{\gamma(t)}{\gamma(s)}\vmu(s) -
    \sum_{m=0}^{k-1}\gamma(t) 
   \hat\ve^{(m)}(\lambda(s)) \int_{\lambda(s)}^{\lambda(t)} 
    e^{-\lambda} 
     \frac{(\lambda - \lambda(s))^m}{m!} \dif \lambda + O((\lambda - \lambda(s))^{k+1}).
    \label{eqn:change-of-var-ode-cont-final}
\end{align}
Given an initial value $\vmu_0$ and time steps $\{t_i\}_{i=0}^{n}$ from $t_0 = 1-\eta$ to $t_n = 0$. The solver uses $n$ steps to iteratively compute a sequence $\{\vmu_{i}\}_{i=1}^{n}$ to approximate the true solutions at time $\{t_i\}_{i=1}^n$ using Eq.~\eqref{eqn:change-of-var-ode-cont-final} by setting $t = t_i$ and $s = t_{i-1}$ for each $i = 1, \cdots, n$.
We can use the well-established finite difference method or the Runge-Kutta method to avoid the computation of the high-order derivatives $\hat \ve^{(m)}(\lambda)$, as we will illustrate in the case of discrete data in Appendix.~\ref{app:discrete-bfn} (see also, e.g., \citet{kloeden1992stochastic,lu2022dpm}).

\paragraph{BFN as a DM with a special noise schedule} The derivation above with a noise prediction network are presented mainly to illustrate the idea. In our implementation, we use a data prediction network $\hat \rvx$ as defined in Appendix.~\ref{append:ancestral_sampling}. 
For simplicity, we do not derive the sampler step by step. Instead, a more straightforward way to get BFN-Solver on continuous data is to treat BFN as a DM with a special noise schedule $\alpha(t) = \gamma(t)$ and $\sigma^2(t) = \gamma(t)(1 - \gamma(t))$. We can directly plugin the relation to the samplers proposed in~\citet{lu2022dpm,lu2022dpm++} to obtain our first-order ODE solver BFN-Solver++1, second-order ODE solver BFN-Solver++2 and second-order SDE solver SDE-BFN-Solver++2, as presented in Algorithm~\ref{ag:BFN-Solver-1-cts},~\ref{ag:BFN-Solver-2-cts} and~\ref{ag:SDE-BFN-Solver-2-cts} respectively. 

The plugging idea does not apply to discrete data and we provide a detailed derivation in Appendix.~\ref{app:discrete-bfn}.

\begin{algorithm}[tb]
   \caption{BFN-Solver++1 (on continuous data)}
   \label{ag:BFN-Solver-1-cts}
\begin{algorithmic}
   \STATE {\bfseries Require:} time steps $\{t_i\}_{i=0}^{M}$ from $t_0 = 1- \eta$ to $t_M = 0$, noise prediction model $\hat{\vepsilon}(\vmu, t)$
   \STATE Denote $\bar{\alpha}_{t} = \gamma(t)$, $\bar{\sigma}_{t}= \sqrt{\gamma(t)(1-\gamma(t))}$, $\lambda_t = \log \frac{\bar \alpha_t}{\bar{\sigma}_t}$
   \STATE ${\vmu}_0 \sim \mathcal{N}(\textbf{0}, \gamma(t_0)(1-\gamma(t_0))\mI), K\beta(t_0)\mI)$
   \FOR{$i=1$ {\bfseries to} $M$}
   \STATE $h_i = (\lambda_{t_i} - \lambda_{t_{i-1}})$
   \STATE $\hat{\rvx}_i = \frac{\vmu_{i-1}}{\gamma(t_{i-1})} - \sqrt{\frac{1-\gamma(t_{i-1})}{\gamma(t_{i-1})}} \hat{\vepsilon}(\vmu_{i-1}, t_{i-1})$
   \STATE ${\vmu}_i = \frac{\bar{\sigma}_{t_i}}{\bar{\sigma}_{t_{i-1}}}{\vmu}_{i-1} - \bar{\alpha}_{t_i}(e^{-h_i} - 1) \hat{\rvx}_i$
   \ENDFOR
   \STATE {\bfseries return} $\hat {\rvx}_M$
\end{algorithmic}
\end{algorithm}
\begin{algorithm}[tb]
   \caption{BFN-Solver++2 (on continuous data)}
   \label{ag:BFN-Solver-2-cts}
\begin{algorithmic}
   \STATE {\bfseries Require:} time steps $\{t_i\}_{i=0}^{M}$ from $t_0 = 1- \eta$ to $t_M = 0$, noise prediction model $\hat{\vepsilon}(\vmu, t)$
   \STATE Denote $\bar{\alpha}_{t} = \gamma(t)$, $\bar{\sigma}_{t}= \sqrt{\gamma(t)(1-\gamma(t))}$, $\lambda_t = \log \frac{\bar {\alpha}_t}{\bar{\sigma}_t}$
   \STATE ${\vmu}_0 \sim \mathcal{N}(\textbf{0}, \gamma(t_0)(1-
   \gamma(t_0))\mI), K\beta(t_0)\mI)$, Initialize an empty buffer $Q$.
   \STATE $\hat{\rvx}_{0} = \frac{\vmu_{0}}{\gamma_{t_{0}}} - \sqrt{\frac{1-\gamma_{t_0}}{\gamma_{t_0}}} \hat{\vepsilon}(\vmu_{0}, t_0)$
   \STATE $Q \leftarrow \hat{\rvx}_{0}$
   \STATE $h_1 = (\lambda_{t_1} - \lambda_{t_{0}})$
   \STATE $\vmu_{1} = \frac{\bar{\sigma}_{t_1}}{\bar{\sigma}_{t_0}}\vmu_{0} - \bar{\alpha}_{t_1}(e^{-h_1}-1) \hat{\rvx}_{0}$
    \STATE $\hat{\rvx}_{1} = \frac{\vmu_{1}}{\gamma_{t_{1}}} - \sqrt{\frac{1-\gamma_{t_1}}{\gamma_{t_1}}} \hat{\vepsilon}(\vmu_{1}, t_1)$
   \STATE $Q \leftarrow \hat{\rvx}_{1}$
   \FOR{$i=2$ {\bfseries to} $M$}
   \STATE $h_i = (\lambda_{t_i} - \lambda_{t_{i-1}})$
   \STATE $r_i = \frac{h_{i-1}}{h_i}$
   \STATE $\mD_i = \left(1+ \frac{1}{2r_i}\right)\hat{\rvx}_{{i-1}} - \frac{1}{2r_i}\hat{\rvx}_{i-2}$
   \STATE ${\vmu}_i = \frac{\bar{\sigma}_{t_i}}{\bar{\sigma}_{t_{i-1}}} - \bar{\alpha}_{t_i}(e^{-h_i} - 1)\mD_i$
   \STATE $\hat{\rvx}_{i} = \frac{\vmu_{i}}{\gamma_{t_{i}}} - \sqrt{\frac{1-\gamma(t_i)}{\gamma(t_i)}} \hat{\vepsilon}(\vmu_{i}, t_i)$
   \STATE If $i<M$, then $Q \leftarrow \hat{\rvx}_{i}$
   \ENDFOR
   \STATE {\bfseries return} $\hat{\rvx}_M$
\end{algorithmic}
\end{algorithm}
                          
\begin{algorithm}[tb]
   \caption{SDE-BFN-Solver++2 (on continuous data)}
   \label{ag:SDE-BFN-Solver-2-cts}
\begin{algorithmic}
   \STATE {\bfseries Require:} time steps $\{t_i\}_{i=0}^{M}$ from $t_0 = 1- \eta$ to $t_M = 0$, noise prediction model $\hat{\vepsilon}(\vmu, t)$
   \STATE Denote $\bar{\alpha}_{t} = \gamma(t)$, $\bar{\sigma}_{t}= \sqrt{\gamma(t)(1-\gamma(t))}$, $\lambda_t = \log \frac{\bar \alpha_t}{\bar{\sigma}_t}$
   \STATE ${\vmu}_0 \sim \mathcal{N}(\textbf{0}, \gamma(t_0)(1-
   \gamma(t_0))\mI), K\beta(t_0)\mI)$, Initialize an empty buffer $Q$.
   \STATE $\hat{\rvx}_{0} = \frac{\vmu_{0}}{\gamma_{t_{0}}} - \sqrt{\frac{1-\gamma_{t_0}}{\gamma_{t_0}}} \hat{\vepsilon}(\vmu_{0}, t_0)$
   \STATE $Q \leftarrow \hat{\rvx}_{0}$
   \STATE $h_1 = (\lambda_{t_1} - \lambda_{t_{0}})$
   \STATE $\vmu_{1} = \frac{\bar{\sigma}_{t_1}}{\bar{\sigma}_{t_0}}\vmu_{0} - \bar{\alpha}_{t_1}(e^{-h_1}-1) \hat{\rvx}_{0}$
    \STATE $\hat{\rvx}_{1} = \frac{\vmu_{1}}{\gamma_{t_{1}}} - \sqrt{\frac{1-\gamma_{t_1}}{\gamma_{t_1}}} \hat{\vepsilon}(\vmu_{1}, t_1)$
   \STATE $Q \leftarrow \hat{\rvx}_{1}$
   \FOR{$i=2$ {\bfseries to} $M$}
   \STATE $h_i = (\lambda_{t_i} - \lambda_{t_{i-1}})$
   \STATE $r_1 = \frac{h_{i-1}}{h_i}$
   \STATE $\mD_i = \frac{1}{r_1}(\hat{\rvx}_{i-1} - \hat{\rvx}_{i-2})$
   \STATE $\rvu_{t_{i-1}} \sim \mathcal{N}(\boldsymbol{0}, \mI)$
   \STATE ${\vmu}_i = \frac{\bar{\sigma}_{t_i}}{\bar{\sigma}_{t_{i-1}}}e^{-h_i}\vmu_{i-1} + \bar{\alpha}_{t_i}(1-e^{-2h_i} )\hat{\rvx}_{i-1} + \frac{\bar{\alpha}_t(1-e^{-2h_i})}{2}\mD_i + \bar{\sigma}_t \sqrt{1-e^{-2h_i}} \rvu_{t_{i-1}}$
   \STATE $\hat{\rvx}_{i} = \frac{\vmu_{i}}{\gamma_{t_{i}}} - \sqrt{\frac{1-\gamma(t_i)}{\gamma(t_i)}} \hat{\vepsilon}(\vmu_{i}, t_i)$
   \STATE If $i<M$, then $Q \leftarrow \hat{\rvx}_{i}$
   \ENDFOR
   \STATE {\bfseries return} $\hat{\rvx}_M$
\end{algorithmic}
\end{algorithm}

\section{Derivation for Continuous-time BFN on Discrete Data}
\label{app:discrete-bfn}
\subsection{Proof of Theorem~\ref{thm:dbfn-sde}}
\label{append:dbfn_sde}
\begin{proof}
 %   In this proof, we will show that the marginal distribution $\rvz(t)$ \emph{conditioned on $\rvx$} of Eq.~\eqref{eq:y_sde-main} at time $t$ is the BFN in Eq.~\eqref{eq:bg-dbfn} 
 %   by the ``variation of constants'' formula~\citep{atkinson2009numerical} and the It\^o's formula~\citep[Theorem 4.2.1]{oksendal2003stochastic}. 
%
 %   Similar to the continuous case, to find the marginal distribution of $\rvz(t)$, we introduce a new process $\tilde \rvz(t) := e^{-\int_0^t H(\tau)\dif \tau} \rvz(t)$.
 %   Using It\^o's formula, we see that $\tilde \rvz$ follows the equation below.
 %   \begin{align*}
 %       \dif \tilde \rvz
 %       &= e^{-\int_0^t H(\tau)\dif \tau} \dif \rvz(t) +  \frac{\dif}{\dif t} e^{-\int_0^t H(\tau)\dif \tau} \rvx(t) \dif t \\
 %       &= e^{-\int_0^t H(\tau)\dif \tau} (H(t) \rvz \dif t + L(t) \dif \vw) - H(t) e^{-\int_0^t H(\tau)\dif \tau} \rvz(t) \dif t \\
 %       &= e^{-\int_0^t H(\tau)\dif \tau} L(t) \dif \vw.
 %   \end{align*}
 %   Written the above equation in the integral form we find that 
 Similar to the proof in Appendix.~\ref{proof:sde}, let $\tilde \rvz(t) := e^{-\int_0^t H(\tau) \dif \tau} \rvz(t)$, then we have 
    \begin{equation}
        \tilde \rvz(t)
       = \tilde \rvz(0) 
        +  \int_0^t e^{-\int_0^s H(\tau)\dif \tau} L(s) \dif \vw(s).
%        \label{eqn:change-of-variable-tilde-mu}
    \end{equation}
%    It is known that from It\^o's isometry~\citep[Corollary 3.1.7]{oksendal2003stochastic} the above It\^o's integral is a Gaussian distribution with variance 
%        $\int_0^t e^{-\int_0^s 2H(\tau)\dif \tau} L(s)^2 \dif s$.
%
    Since $H(t) = \frac{\dif}{\dif t} \ln {\beta(t)}$ and $L(t)^2 = -K \frac{\dif}{\dif t} \beta(t)$, then the above equation becomes
    \begin{equation*}
        \frac{\beta(0)}{\beta(t)} \rvz(t) \sim \rvz(0) + \mathcal N\left( 0, K\beta(0)^2\left( \frac{1}{\beta(t)} - \frac{1}{\beta(0)} \right) \mI \right),
    \end{equation*}
    which shows that the distribution of $\rvz(t)$ conditioned on $\rvz(0)$ is
    \begin{equation}
        \label{eqn:dis-cond-z0}
        \rvz(t) \mid \rvz(0) \sim \mathcal N\left( \frac{\beta(t)}{\beta(0)} \rvz(0), K\left (\beta(t) - \frac{\beta(t)^2}{\beta(0)} \right ) \mI\right).
    \end{equation}
    Recall that from Eq.~\eqref{eq:bg-dbfn} with $t = 0$ we know that
    \begin{equation}
        \label{eqn:dis-cond-x}
        \rvz(0) \mid \rvx \sim \mathcal N(\beta(0) \vw_\rvx, K \beta(0) \mI).
    \end{equation}
    Combining the above two equations, we find that 
    \begin{equation}
        \rvz(t) \mid \rvx \sim \mathcal N(\beta(t) \vw_\rvx, K\beta(t)\mI),
    \end{equation}
    which is the BFN in Eq.~\eqref{eq:bg-dbfn} for any $t \in [0, 1 -\eta]$.
\end{proof}

\subsection{Proof of Theorem~\ref{thm:discrete-dsm}}
\label{app:discrete-dsm}
\begin{proof}
Recall that the loss function on discrete data in Eq.~\eqref{eq:bg-dbfn-train} is
\begin{align}
\mathcal{L}^{\infty}(\rvx)
    &= \E_{\flow(\vtheta|\rvx, t), t\sim U(0, 1)}K\beta_1t  \| \ve_\rvx - \hat{\ve}(\vtheta, t)\|^2,
\end{align}
where $\flow(\vtheta|\rvx, t)$ is specially defined with softmax function to ensure $\vtheta$ lies in the simplex as follows, making its density complex.\footnote{Note that $\vtheta$ lies in a low-dimensional simplex, so the ``density'' should be defined w.r.t. a low-dimensional Lebesgue measure. We intend to define the distribution $q_F$ by a sampling procdure as introduced later.}
\begin{align*}
    \flow(\vtheta|\rvx, t) = \E_{\rvz(t) \sim \mathcal{N}(\beta(t)(K\ve_{\rvx}- \boldsymbol{1}), \beta(t)K\mI)}\delta(\vtheta - \text{softmax}(\rvz(t))).
\end{align*}

Note that to obtain a sample $\vtheta(t) \sim q_F(\vtheta|\vx, t)$, we can first sample $\rvz(t) \sim \mathcal N(\beta(t)(K\ve_{\rvx} - 1), \beta(t)K\mI)$ and apply the deterministic transform $\vtheta(t) := \text{softmax}(\rvz(t))$.
Then, we can rewrite the loss function as
\begin{align*}
    \mathcal{L}^{\infty}(\rvx)
    &= \E_{t \sim U(0, 1)}\E_{\rvz(t) \sim q(\rvz(t)|\rvx)}K\beta_1t  \| \ve_\rvx - \hat{\ve}(\text{softmax}(\rvz(t)), t)\|^2 \\
    &= \E_{t \sim U(0, 1)}\E_{\rvz(t) \sim q(\rvz(t)|\rvx)}K\beta_1t  \| \ve_\rvx - \hat{\ve}_s(\rvz(t), t)\|^2,
\end{align*}
where $q(\rvz(t) | \rvx) = \mathcal{N}(\beta(t)(K\ve_\rvx-1), \beta(t)K\mI)$ whose score function is given by 
\begin{align}
    \nabla_\rvz \log q(\rvz(t) | \rvx) = -
    \frac{\rvz(t)}{K\beta(t)} + \ve_\rvx - \frac{1}{K}.
\end{align}
Let
\begin{align}
\label{eqn:one-to-one-correspondance-of-discrete-loss}
    \hat{\vs}(\rvz(t), t) := -
    \frac{\rvz(t)}{K\beta(t)} + \hat{\ve}_s(\rvz(t), t) - \frac{1}{K}.
\end{align}    
Substituting $\nabla_\rvz \log q(\rvz(t) | \rvx)$ and $\hat{\vs}(\text{softmax}(\rvz(t)), t)$ into the loss function yields
\begin{align}
    \mathcal{L}^{\infty}(\rvx)
    &= \E_{t \sim U(0, 1)}\E_{\rvz(t) \sim q(\rvz(t) | \rvx)}K\beta_1t  \| \ve_\rvx - \hat{\ve}_s(\rvz(t), t)\|^2 \nonumber \\
    &= \E_{t \sim U(0, 1)}\E_{\rvz(t) \sim q(\rvz(t) | \rvx)}K\beta_1t\| \nabla_\rvz \log q(\rvz(t) | \rvx) - \hat{\vs}(\rvz(t), t)\|^2. \label{eqn:discrete-score-matching-loss}
\end{align}

The loss in Eq.~\eqref{eqn:discrete-score-matching-loss} performs the denosing score matching (DSM), and according to~\citet{vincent2011connection}, the optimal solution of it w.r.t. $\hat \vs$ is the score of the distribution of $\rvz(t)$.
Inspecting Eq.~\eqref{eqn:one-to-one-correspondance-of-discrete-loss}, we find that there is a one-to-one correspondance between the parameterized score function $\hat \vs$ in the DSM loss and the function $\hat \ve_s$ in BFN, showing that the optimization of BFN is equivalent to that of DM. 
\end{proof}

\subsection{Proof of Proposition~\ref{thm:dbfn-sample}}
\label{app:dis_bfn_sampler}
\begin{proof}
Recall that the BFN sampler is defined in Eqs.~\eqref{eq:bg-dbfn-cat-sample} and \eqref{eq:bg-dbfn-main-sample}, and $\beta(t) = \beta_1(1 - t)^2$. 
If we remove the categorical sampling step in Eq.~\eqref{eq:bg-dbfn-cat-sample}, the sampling rule becomes
\begin{align}
    \rvz_i &= \rvz_{i-1} + (\beta(t_i) - \beta(t_{i-1})) (K \hat \ve_s(\rvz_{i-1}, t_{i-1}) - 1) + \sqrt{K(\beta(t_i) - \beta(t_{i-1}))} \rvu_i,  \label{eqn:bfn-update-without-cat}
\end{align}
where $\rvu_i \sim \mathcal{N}(\boldsymbol{0}, \mI)$. 
Next, to show that the update rule in Eq.~\eqref{eqn:bfn-update-without-cat} is a first-order discretization of the reverse SDE, 
we write the SDE defined in Eq.~\eqref{eq:dbfn-para-reverse-sde} in the integral form as follows for any $0 \leq t < s \leq 1 - \eta$:
\begin{align}
    \rvz(t) = \rvz(s) - \int_s^t L(\tau)^2 \left [ \hat{\ve}_s(\rvz(\tau), \tau) - \frac{1}{K} \right ] \dif \tau + \int_{s}^t L(\tau) \dif \bar \vw(\tau).
    \label{eqn:discrete-equation-integral-form}
\end{align}
Note that the It\^o integral follows the Gaussian distribution, and by Eq.~\eqref{eqn:discrete-diffusion-coeff}, we know that 
\begin{equation*}
    \int_s^t L(\tau) \dif \bar \vw(\tau) 
    \sim \mathcal N\left ( \boldsymbol{0}, \int_t^s L(\tau)^2 \dif t \mI\right )
    = \mathcal N\left ( \boldsymbol{0}, K (\beta(t)- \beta(s))\mI \right ).
\end{equation*}
Then, by approximating $\hat \ve_s(\rvz(\tau), \tau) = \hat \ve_s(\rvz(s), s) + O(\tau - s)$ in Eq.~\eqref{eqn:discrete-equation-integral-form}, we find 
\begin{align*}
    \rvz(t) &= \rvz(s) - \int_s^t L(\tau)^2 \left [  \hat \ve_s(\rvz(s), s) + O(\tau- s) - \frac{1}{K} \right ] \dif \tau + \int_s^t L(\tau) \dif \bar \vw(\tau) \\
    &= \rvz(s)     + K (\beta(t) - \beta(s)) \left [  \hat \ve_s(\rvz(s), s) + O(t - s) - \frac{1}{K} \right ] + \sqrt{K(\beta(t) - \beta(s))} \rvu_s \\
    &= \rvz(s)     + (\beta(t) - \beta(s)) \left [  K\hat \ve_s(\rvz(s), s) - 1 \right ] + \sqrt{K(\beta(t) - \beta(s))} \rvu_s + O((t - s)^2),
\end{align*}
where $\rvu_s \sim \mathcal{N}(\boldsymbol{0},\mI)$. Now, the proof is completed by setting $t = t_{i}$ and $s = t_{i-1}$ in the above equation, and comparing the result with Eq.~\eqref{eqn:bfn-update-without-cat}.
\end{proof}
%After simplifying, we get 
\subsection{Derivation of SDE-BFN-Solvers on Discrete Data}
In this section we derive the SDE-BFN-Solvers of solving Eq.~\eqref{eq:dbfn-para-reverse-sde}.
We shall highlight that the SDE to be solved is different from those solved in BFNs and DMs on continuous data, as it contains no linear terms.
We also follow the recipe of DPM-Solvers~\citep{lu2022dpm,lu2022dpm++} in reducing the discretization error to analytically simplify the equation as much as possible, instead of approximating it directly.
The derived algorithms can be found in Algorithm~\ref{ag:SDE-BFN-Solver-1-dis},~\ref{ag:SDE-BFN-Solver-2-dis}.

As we have seen in Appendix.~\ref{app:dis_bfn_sampler}, the BFN sampler without the categorical sampling step (Eq.~\eqref{eqn:bfn-update-without-cat}) is the desired first-order solver, which is called SDE-BFN-Solver1.
To further reduce the discretization error, we consider the first-order approximation of $\hat \ve_s(\rvz(\tau), \tau)$ in Eq.~\eqref{eqn:discrete-equation-integral-form}:
\begin{equation}
    \hat \ve_s(\rvz(\tau), \tau) = \hat \ve_s(\rvz(s), s) + \hat \ve_s^{(1)}(\rvz(s), s) (\tau - s) + O((\tau - s)^2),
\end{equation}
where $\hat \ve_s^{(1)}$ is the derivative of $\hat \ve_s(\rvz(s), s)$ w.r.t. $s$.
Then, setting $t = t_i$ and $s = t_{i-1}$, Eq.~\eqref{eqn:discrete-equation-integral-form} becomes
\begin{align*}
    \rvz(t_{i}) &= \rvz(t_{i-1}) - \int_{t_{i-1}}^{t_i} L(\tau)^2 \left [  \hat \ve_s(\rvz(t_{i-1}), t_{i-1}) +\hat \ve_s^{(1)}(\rvz(t_{i-1}), t_{i-1})(\tau - t_{i-1})  - \frac{1}{K} \right ] \dif \tau 
    \\
    &\phantom{{}={}}+ \int_{t_{i-1}}^{t_i} L(\tau) \dif \bar \vw(\tau)+O((t_i- t_{i-1})^3)  \\
    &= \rvz(t_{i-1})     + (\beta({t_i}) - \beta(t_{i-1})) \left [  K\hat \ve_s(\rvz(t_{i-1}), t_{i-1}) - 1 \right ] 
    \\
    &\phantom{{}={}}- \frac{1}{3}K\beta_1({t_i} - t_{i-1})^2(t_{i-1} + 2{t_i} - 3) \hat \ve^{(1)}_s(\rvz(t_{i-1}), t_{i-1})  \\
    &\phantom{{}={}}+ \sqrt{K(\beta(t_i) - \beta(t_{i-1}))} \mathcal N(0, \mI) + O((t_i - t_{i-1})^3).
\end{align*}
Finally, we approximate the derivative $\hat \ve_s^{(1)}$ by a finite difference as follows to obtain the final discretization algorithm, namely SDE-DPM-Solver2.
\begin{equation}
\hat \ve_s^{(1)}(\rvz_{i-1}, t_{i-1}) \approx 
\frac{\hat \ve_s(\rvz_{i-2}, t_{i-2}) - \hat \ve_s(\rvz_{i-1}, t_{i-1})}{t_{i-2} - t_{i-1}}.
\end{equation}

\subsection{Derivation of BFN-Solvers on Discrete Data}
Eq.~\label{app:dis_bfn_solver}
In this section, we derive the BFN-Solvers on discrete data using the following integral form of the parameterized ODE in~\eqref{eq:dbfn-para-ode}.
\begin{align}
\label{eqn:discrete-bfn-integral-form}
    \rvz(t) = \frac{1-t}{1-s}\rvz(s) + \beta_1(1-t)(t-s) - K\beta_1(1-t)\int_s^t \hat{\ve}_s(\rvz(\tau), \tau) \dif \tau.
\end{align}
\paragraph{BFN-Solver1}
The first-order solver approximates $\hat \ve_s(\rvz(\tau), \tau)$ in the above integral by $\hat \ve_s(\rvz(s), s)$ directly, yielding the following update rule:
\begin{equation}
\label{eqn:bfn-solver-1}
    \rvz_i = \frac{1-t_i}{1-t_{i-1}}\rvz_{i-1} + \beta_1(1-t_i)(t_i-t_{i-1})(1 - K\hat\ve_s(\rvz_{i-1}, t_{i-1})).
\end{equation}
%
%
%We need to approximate the integral of $\hat{\ve}_s(\rvz_\tau, \tau)$. The $k-1$-th order Taylor expansion of $\hat{\ve}_s(\rvz_\tau, \tau)$ at $s$ is
%\begin{align}
%    \hat{\ve}_s(\rvz_\tau, \tau) =  \sum_{n=0}^{k-1} \frac{\hat{\ve}_s^{(n)}(\rvz(s), 1-s)}{n!} + O((t-s)^k).
%\end{align}
%
%When We choose $k = 1$, the integral term is:
%\begin{align}
%    \int_s^t \hat{\ve}_s(\rvz_\tau, 1-\tau) \dif \tau  
%    &= \int_s^t \hat{\ve}_s(\rvz(s), 1-s) + O(\tau-s) \dif \tau\\
%    &= \hat{\ve}_s(\rvz(s), 1-s) (t - s) + O((t-s)^2)
%\end{align}
%Let $\eta > 0$ be an arbitrarily small constant. Let time steps $\{t_i\}_{i=0}^{M}$ from $t_0 = 1- \eta$ to $t_M = 0$. Given an initial value sample from an estimated distribution $(\Tilde{\rvz})_0 \sim \Tilde{p}(\rvz(1-\eta)) = \mathcal{N}(\textbf{0}, K\beta(1-\eta)\mI)$.  $\{\Tilde{\rvz}_i\}_{i=1}^M$ is computed iteratively as follows:
%\begin{dmath}
%    \Tilde{\rvz}_i = \frac{1-t_i}{1-t_{i-1}}\Tilde{\rvz}_{i-1} + \beta_1(1-t_i)(t_i-t_{i-1})(1 - K\hat{\ve}(\Tilde{\rvz}_{i-1}, 1-t_{i-1})).
%\end{dmath}
%
\paragraph{BFN-Solver2}
The second-order solver approximates $\hat \ve_s(\rvz(\tau), \tau)$  with $\hat \ve_s(\rvz(s), s) + \hat \ve_s^{(1)}(\rvz(s), s)(\tau - s)$, where $\hat \ve_s(\rvz(s), s)$ is the derivative of $\hat \ve_s$ w.r.t. $s$.
Using such an approximation, we have
\begin{align}
    \int_s^t \hat{\ve}_s(\rvz(\tau), \tau) \dif \tau &= \int_s^t \hat{\ve}_s(\rvz(s), s) + \hat{\ve}_s^{(1)}(\rvz(s), s) (\tau - s) + O(\tau-s)^2  \dif \tau \nonumber \\
    &= \hat{\ve}_s(\rvz(s), s)(t-s) + \hat{\ve}_s^{(1)}(\rvz(s), s) \frac{(t-s)^2}{2} + O((t-s)^3). \label{eqn:bfn-solver2-approx}
\end{align}
%where $\hat{\ve}_s^{(1)}(\rvz(s), 1-s)$ can be approximate with additional intermediate point. Let $ s < \tau < t$:
Following the method used in DPM-Solver~\citep{lu2022dpm}, we use an intermediate time $r\in (s, t)$ to approximate the derivative term.
We use the first-order approximation in Eq.~\eqref{eqn:bfn-solver-1} to compute $\rvz(r)$ and the finite difference to estimate the derivative $\hat \ve_s^{(1)}$:
\begin{align}
    \rvz(r)&= \frac{1-t}{1-s}\rvz(s) + \beta_1(1-r)(r-s)(1 - K\hat{\ve}_s(\rvz(s), s)) + O((r - s)^2),\\
    \hat{\ve}_s^{(1)}(\rvz(s), s) &= \frac{\hat{\ve}_s(\rvz(r), r) - \hat{\ve}_s(\rvz(s), s)}{ r - s} + O(r - s).
\end{align}
Combining the above equations with Eq.~\eqref{eqn:bfn-solver2-approx}, we find that
%Based on that the integral term is
\begin{align}
    \int_s^t \hat{\ve}_s(\rvz(\tau), \tau) \dif \tau 
    &= \hat{\ve}_s(\rvz(s), s)(t-s) + \frac{(\hat{\ve}_s(\rvz(r), r) - \hat{\ve}_s(\rvz(s), s))(t-s)^2}{2( r - s)}+ O((t-s)^3).
\end{align}
Finally, let $\eta > 0$ be an arbitrarily small constant and choose time steps $\{t_i\}_{i=0}^{M}$ from $t_0 = 1- \eta$ to $t_M = 0$. Given an initial value sample, $\{\rvz_i\}_{i=1}^M$ is computed iteratively as follows, by choosing $(s, t, r) := \left(t_{i-1}, t_i, \frac{t_i + t_{i-1}}{2}\right)$ for each $i$ in the above derivation.
\begin{align}
    \rvz_{i-1/2} &= \frac{1-t_{i-1/2}}{1-t_{i-1}}\rvz_{i-1} + \beta_1(1-t_{i-1/2})(t_{i-1/2}-t_{i-1})(1 - K\hat{\ve}_s(\rvz_{i-1}, t_{i-1})),\\
    \rvz_i &= \frac{1-t_i}{1-t_{i-1}}\rvz_{i-1} + \beta_1(1-t_i)(t_i-t_{i-1})  
    - c(t_i)(t_i - t_{i-1}) \hat \ve_s(\rvz_{i-1}, t_{i-1}) \nonumber \\
    &\phantom{{}={}}- c(t_i) \frac{(t_i - t_{i-1})^2}{2(t_{i - 1/2} - t_{i-1})} (\hat\ve_s(\rvz_{i - 1/2}, t_{i - 1/2}) - \hat \ve_s(\rvz_{i-1}, t_{i-1})).
%    \\
%    \rvz_i &= \frac{1-t_i}{1-t_{i-1}}\rvz_{i-1} + \beta_1(1-t_i)(t_i-t_{i-1}) - c(t_{i})\left[t_{i}-t_{i-1} -\frac{(t_{i}-t_{i-1})^2}{2(t_{i-1/2}-t_{i-1})}\right] \hat{\ve}_s(\rvz_{t_{i-1}}, 1-t_{i-1}) \\
%    &- c(t_{i}) \frac{(t_{i}-t_{i-1})^2}{2(t_{i-1/2}-t_{i-1})}\hat{\ve}_s(\rvz_{t_{i-1/2}}, 1-t_{i-1/2})
\end{align}
where $t_{i-1/2} = (t_i + t_{i-1})/2$ and $c(t) = K\beta_1(1-t)$.

\begin{algorithm}[tb]
   \caption{SDE-BFN-Solver1 (on discrete data)}
   \label{ag:SDE-BFN-Solver-1-dis}
\begin{algorithmic}
   \STATE {\bfseries Require:} time steps $\{t_i\}_{i=0}^{M}$, from $t_0 = 1 - \eta$ to $t_M = 0$, model $\hat{\ve}_s(\rvz, t)$, $\beta(t) = \beta_1(1-t)^2$
   \STATE $\rvz_0 \sim \mathcal{N}(\boldsymbol{0}, K\beta(t_0)\mI)$
   \FOR{$i=1$ {\bfseries to} $M-1$}
   \STATE $\rvu_i \sim \mathcal{N}(\boldsymbol{0}, \mI)$
   \STATE $\rvz_{i} = \rvz_{i-1} + (\beta(t_{i}) - \beta(t_{i-1}))(K\hat{\ve}_s(\rvz_{i-1}, t_{i-1}))-1) + \sqrt{K(\beta(t_{i}) - \beta(t_{i-1}))} \rvu_i$
   \ENDFOR
   \STATE $\hat {\rvx} = \text{argmax}(\hat{\ve}_s(\rvz_{M-1}, t_{M-1}))$
   \STATE {\bfseries return} $\hat {\rvx}$
\end{algorithmic}
\end{algorithm}
\begin{algorithm}[tb]
   \caption{SDE-BFN-Solver2 (on discrete data)}
   \label{ag:SDE-BFN-Solver-2-dis}
\begin{algorithmic}
   \STATE {\bfseries Require:} time steps $\{t_i\}_{i=0}^{M}$, from $t_0 = 1 - \eta$ to $t_M = 0$, model $\hat{\ve}_s(\rvz, t)$, $\beta(t) = \beta_1(1-t)^2$
   \STATE $\rvz_0 \sim \mathcal{N}(\boldsymbol{0}, K\beta(t_0)\mI)$, Initialize an empty buffer $Q$
   \STATE $\rvu_1 \sim \mathcal{N}(\boldsymbol{0}, \mI)$
   \STATE $Q \leftarrow \hat{\ve}_s(\rvz_{0}, t_{0})$
   \STATE $\rvz_{1} = \rvz_{0} + (\beta(t_{1}) - \beta(t_{0}))(K\hat{\ve}_s(\rvz_{0}, t_{0}))-1) + \sqrt{K(\beta(t_{1}) - \beta(t_{0}))} \rvu_i$
   \FOR{$i=2$ {\bfseries to} $M-1$}
   \STATE $D_1 = \frac{\hat{\ve}_s(\rvz_{i-2}, t_{i-2}) -\hat{\ve}_s(\rvz_{i-1}, t_{i-1})}{t_{i-2} - t_{i-1}}$
   \STATE $\rvu_i \sim \mathcal{N}(\boldsymbol{0}, \mI)$
   \STATE $\rvz_{i} = \rvz_{i-1} + (\beta(t_{i}) - \beta(t_{i-1}))(K\hat{\ve}_s(\rvz_{i-1}, t_{i-1}))-1) - \frac{1}{3}K\beta_1(t_i-t_{i-1})^2 (t_{i-1} + 2t_i -3) D_1 + \sqrt{K(\beta(t_{i}) - \beta(t_{i-1}))} \rvu_i$
   \STATE If $i<M-1$, then $Q \leftarrow \hat{\ve}_s(\rvz_{i-1}, t_{i-1})$
   \ENDFOR
   \STATE $\hat {\rvx} = \text{argmax}(\hat{\ve}_s(\rvz_{M-1}, t_{M-1}))$
   \STATE {\bfseries return} $\hat {\rvx}$
\end{algorithmic}
\end{algorithm}
\begin{algorithm}[tb]
   \caption{BFN-Solver1 (on discrete data)}
   \label{ag:BFN-Solver-1-dis}
\begin{algorithmic}
   \STATE {\bfseries Require:} time steps $\{t_i\}_{i=0}^{M}$, from $t_0 = 1 - \eta$ to $t_M = 0$, model $\hat{\ve}_s(\rvz, t)$, $\beta(t) = \beta_1(1-t)^2$
   \STATE $\rvz_0 \sim \mathcal{N}(\boldsymbol{0}, K\beta(t_0)\mI)$
   \FOR{$i=1$ {\bfseries to} $M-1$}
   \STATE $\Tilde{\rvz}_i = \frac{1-t_i}{1-t_{i-1}}\rvz_{i-1} + \beta_1(1-t_i)(t_i-t_{i-1})(1 - K\hat{\ve}_s({\rvz}_{i-1}, t_{i-1}))$
   \ENDFOR
   \STATE $\hat {\rvx} = \text{argmax}(\hat{\ve}_s(\rvz_{M-1}, t_{M-1}))$
   \STATE {\bfseries return} $\hat {\rvx}$
\end{algorithmic}
\end{algorithm}
\begin{algorithm}[tb]
   \caption{BFN-Solver2 (on discrete data)}
   \label{ag:BFN-Solver-2-dis}
\begin{algorithmic}
   \STATE {\bfseries Require:} time steps $\{t_i\}_{i=0}^{M}$, from $t_0 = 1 - \eta$ to $t_M = 0$, model $\hat{\ve}_s(\rvz, t)$, $\beta(t) = \beta_1(1-t)^2$ 
   \STATE $\rvz_0 \sim \mathcal{N}(\boldsymbol{0}, K\beta(t_0)\mI)$
   \FOR{$i=1$ {\bfseries to} $M-1$}
    \STATE $t_{i-1/2} = (t_i + t_{i-1})/2$
    \STATE $c_i = K\beta_1(1-t_i)$
    \STATE $D_1 = \frac{\hat{\ve}_s({\rvz}_{i-1/2}, t_{i-1/2}) - \hat{\ve}_s({\rvz}_{i-1}, t_{i-1})}{t_{i-1/2} - t_{i-1}}$
    \STATE $\rvz_i = \frac{1-t_i}{1-t_{i-1}}\rvz_{i-1} + \beta_1(1-t_i)(t_i-t_{i-1}) - c_i(t_i-t_{i-1})\hat{\ve}_s({\rvz}_{i-1}, t_{i-1}) -  \frac{c_i(t_i-t_{i-1})^2}{2}D_1$
   \ENDFOR
   \STATE $\hat {\rvx} = \text{argmax}(\hat{\ve}_s(\rvz_{M-1}, t_{M-1}))$
   \STATE {\bfseries return} $\hat {\rvx}$
\end{algorithmic}
\end{algorithm}

\section{Experimental Details}
\label{app:experimental_details}

\subsection{Choices of the Initialization Distribution} \label{sec:optimal-init}

Since the exact distribution of $\vmu(1 - \eta)$ is unknown, we need to choose an approximation of it as the initialization distribution.
The following proposition identifies the best initial distribution among isotropic Gaussian distributions.
\begin{proposition}
    Let $p_t(\vmu) = \E_{\rvx \sim p_{data}} q_F(\vmu|\rvx, \gamma(t))$ be the distribution of $\vmu(t)$ in Eq.~\eqref{eq:bg-cbfn}, then $q_t(\vmu) := \mathcal N(\vmu | \vm^*(t), (\sigma^*(t))^2 \mI)$ minimizes the Kullback-Leibler (KL) divergence between $p_t$ and any isotropic Gaussian distributions, where
    \begin{equation*}
        \vm^*(t) = \gamma(t) \vm_{data}, \quad \text{and} \quad
        (\sigma^*(t))^2 
        = \gamma(t)(1 - \gamma(t)) + \gamma(t)^2 \frac{\mathrm{Tr}\,\Sigma_{data}}{D},
    \end{equation*}
    $D$ is the dimensionality of $\vmu$, and $\vm_{data}, \Sigma_{data}$ are the mean and the covariance matrix of the data distribution $p_{data}$, respectively.
    In other words, we have
    \begin{equation}
        (\vm^*(t), \sigma^*(t)) = \argmin_{\vm, \sigma} D_{\rm KL}(p_t \| \mathcal N(\vm, \sigma^2 \mI)), 
    \end{equation}
    where $D_{\rm KL}$ is the Kullback-Leibler divergence.
\end{proposition}
\begin{proof}
    Since $p_F(\vmu(t) | \rvx, \gamma(t)) = \mathcal N(\gamma(t) \rvx, \gamma(t)(1 - \gamma(t))\mI)$, the KL divergence can be written as 
    \begin{align*}
        D_{\rm KL}(p_t \| \mathcal N(\vm, \sigma^2 \mI))
        &=  -\E_{\vmu \sim p_t} \log \mathcal N(\vmu | \vm, \sigma^2 \mI) + \E_{\vmu \sim p_t} \log p_t(\vmu) \\
        &=  \E_{\vmu \sim p_t} \left [ \frac{D}{2} \log (2\pi) + D \log \sigma + \frac{\| \vmu - \vm \|^2}{2\sigma^2}  \right ]+ \E_{\vmu \sim p_t} \log p_t(\vmu),
    \end{align*}
    where $D$ is the dimensionality of $\vmu$. By taking the derivative, we know the minimizer is
    \begin{align*}
        \vm^*(t) &= \E_{\vmu \sim p_t} \vmu 
        = \gamma(t) \vm_{data}, \\
        (\sigma^*(t))^2 &= \frac{1}{D} \E_{\vmu \sim p_t} \|\vmu - \vm^*(t)\|^2
        = \gamma(t)(1 - \gamma(t)) + \gamma(t)^2 \frac{\mathrm{Tr}\,\Sigma_{data}}{D}.
    \end{align*}
\end{proof}
However, the optimal distribution in the above proposition depends on the mean and the covariance matrix of the data distribution, which are still unknown.
Fortunately, we find that for sufficiently small $\eta$, the effect of data-dependent terms is negligible. Since when $\eta \to 0$ we know $\gamma(t) \to 0$ and by the proposition the optimal variance is $\gamma(t)(1 - \gamma(t)) + o(\gamma(t))$, so a typical sample from the optimal distribution has the norm $\sqrt{D \gamma(t) (1-\gamma(t))}$,
which dominates the optimal mean $\gamma(t) \vm_{data}$ as $\gamma(t) \to 0$.
Therefore, we suggest the data-independent distribution $\mathcal N(\boldsymbol{0}, \gamma(t)(1 - \gamma(t))\mI))$ as the initial distribution.

We note that although the above proposition is proved for the BFN with continuous data, a similar result also holds for discrete data.

\subsection{Final Step of Sampling}
Theoretically, we need to solve the reverse SDEs or ODEs from time $1-\eta$ to $0$ to generate samples. In practice, the BFN sampler adds a step where it uses the state obtained from the initial $n$ iterations, denoted by $\vtheta_{n}$, to run the network an additional time. The output from this final network is then used as the final sample. We follow this sampling trick as our default sampling method at the final step.

\subsection{User Study for Text Generation}
In this section, we give the details of the user study designed to compare the quality of text generated from BFN-Solvers and BFN from the human perspective. We instruct participants to select the text samples they perceive as higher quality, providing them with authentic examples from the test dataset to serve as benchmarks for their evaluations. For each study, we collect 500 human answers from 10 participants to ensure a sufficient sample size to analyze the preferences.

\subsection{Ablation of Categorical Sampling Step in BFN Sampler on Discrete Data}
\label{app:ablation}
In this section, we present the detailed results of Sec.~\ref{sec:exp_fast_continous} in Table~\ref{tab:without_cs}. We explicitly analyze the sampling dynamics of the original BFN sampler with and without CS. We generate two sets of trajectories, denoted as $\rvz_{\text{w/cat}}(t)$ and $\rvz_{\text{w/o cat}}(t)$, using identical Gaussian noises and then compute the $1$-norm of the difference between these trajectories at each time point $t$ and sum them. We replicate the sampling $1000$ times independently and compute the mean, which represents the disparity between two samplers. As shown in Table~\ref{tab:l1_without_cs}, $\|\rvz_{\text{w/cat}} - \rvz_{\text{w/o cat}}\|_1$ gets smaller as the step size decreases. 

\begin{table}[ht!]
\caption{\textbf{Ablation of the categorical sampling (CS) step in the BFN sampler on discrete data.} We present the SA$\uparrow$.}
\label{tab:without_cs}
\vskip 0.15in
\begin{center}
\begin{small}
\begin{sc}
\begin{tabular}{lccccccccr}
\toprule
NFE                      & 10    & 12    & 20    & 30    & 40    & 50    & 100   & 200   & 1000  \\
\midrule
BFN                      & 66.27 & 70.73 & 79.54 & 82.65 & 84.01 & 84.88 & \textbf{86.14} & 86.50 & 86.69 \\
BFN w/o CS (ours)             & \textbf{70.67} & \textbf{73.63} & \textbf{80.86} & \textbf{83.26} & \textbf{84.27} & \textbf{84.90} & 86.11 & \textbf{86.76} & \textbf{86.77} \\
\bottomrule
\end{tabular}
\end{sc}
\end{small}
\end{center}
\vskip -0.1in
\end{table}

\begin{table}[ht!]
\caption{\textbf{Ablation of the categorical sampling (CS) step in the BFN sampler on discrete data.} We present the $L1$ norm of the difference between sampling trajectories of the original BFN sampler with and without CS, varying the step size.}
\label{tab:l1_without_cs}
\vskip 0.15in
\begin{center}
\begin{small}
\begin{sc}
\begin{tabular}{lccccccr}
\toprule
Step size     & 0.1    & 0.05    & 0.02    & 0.01    & 0.005    & 0.002    &0.001   \\
\midrule
$\|\rvz_{\text{w/cat}} - \rvz_{\text{w/o cat}}\|_1$  & 0.18 & 0.16 & 0.13 & 0.10 & 0.087 & 0.066 & 0.049 \\ 
\bottomrule
\end{tabular}
\end{sc}
\end{small}
\end{center}
\vskip -0.1in
\end{table}

\section{Additional results}

\subsection{Additional Results on Continuous Dataset}
\label{app:results_continuous}
In this section, we present the detailed results of Sec.~\ref{sec:exp_fast_continous}. As shown in Tab.~\ref{tab:fid_nfe}, our proposed methods obtain the best results under all NFE.

\begin{table}[t!]
\caption{\textbf{Image generation results on continuous CIFAR-10 dataset.} Sampling quality is measured by FID $\downarrow$, varying the number of function evaluations (NFE). We \textbf{bold} the best result under the
corresponding setting. For instance, we underline the result of BFN at 50 steps and our BFN-Solvers2++ solver at 10 steps, where we achieve a speed-up of 5 times.}
\label{tab:fid_nfe}
\vskip 0.15in
\begin{center}
\begin{small}
\begin{sc}
\begin{tabular}{lccccccr}
\toprule
NFE                      & 10       & 20       & 50      & 100     & 200    & 500    & 1000  \\ 
\midrule
BFN                      & 253.23   & 146.19   & \underline{74.81}   & 45.50   & 27.98  & 16.24  & 13.05 \\
SDE-BFN-Solver++2 (ours) & {77.19}    & 49.73    & 34.29   & \textbf{23.09}   & \textbf{16.07}  & \textbf{11.94} & \textbf{10.86}  \\
BFN-Solver++1 (ours)     & {66.27}    & 49.18    & 36.27   & 30.43   & 26.62  & 21.49  & 20.39 \\
BFN-Solver++2 (ours)     & \underline{\textbf{55.87}}  & \textbf{43.63}   & \textbf{33.01}   & 27.92   & 23.89  & 19.48 & 20.51 \\
\bottomrule
\end{tabular}
\end{sc}
\end{small}
\end{center}
\vskip -0.1in
\end{table}

\subsection{Additional Results of Discrete Dataset}
\label{app:results_discrete}
In this section, we present the detailed results of Sec.~\ref{sec:exp_fast_discrete}. As shown in Tab.~\ref{tab:spelling_nfe}, our best
solver significantly outperforms the original BFN sampler
with a few (e.g., 10) NFEs under the SA metric.

\begin{table}[t!]
\caption{\textbf{Text generation results on discrete text8 dataset.} Sampling quality is measured by SA $\uparrow$, varying the number of function evaluations (NFE). For instance, we underline the result of BFN at 1000 steps and our SDE-BFN-Solvers2 solver at 50 steps, where we achieve a speed-up of 20 times.} 
\label{tab:spelling_nfe}
\vskip 0.15in
\begin{center}
\begin{small}
\begin{sc}
\begin{tabular}{lccccccccr}
\toprule
NFE                      & 10    & 12    & 20    & 30    & 40    & 50    & 100   & 200   & 1000  \\
\midrule
BFN                      & 66.27 & 70.73 & 79.54 & 82.65 & 84.01 & 84.88 & 86.14 & 86.50 & \underline{86.69} \\
SDE-BFN-Solver2 (ours)  & 80.29 & 82.39 & \textbf{85.27} & \textbf{86.03} & \textbf{86.63} & \underline{\textbf{86.82}} & \textbf{87.04} & \textbf{86.74} & \textbf{86.78} \\
BFN-Solver1 (ours)              & 78.61 & 80.21 & 82.46 & 83.76 & 84.18 & 84.53 & 85.25 & 85.39 & 85.61 \\
BFN-Solver2 (ours)             & \textbf{82.27} & \textbf{83.23} & 84.46 & 85.12 & 85.34 & 85.46 & 85.53 & 85.57 & 85.60 \\
\bottomrule
\end{tabular}
\end{sc}
\end{small}
\end{center}
\vskip -0.1in
\end{table}

\subsection{Analysis of $\eta$ on Continuous Data}
\label{app:analysis_eta_continous}
In this section, we provide an analysis of the hyperparameter $\eta$ for continuous data. According to the analyses in Appendix~\ref{sec:optimal-init}, we start sampling from the approximate prior distribution $\tilde{p}\left(\vmu(1 - \eta) \right) = \mathcal{N}\left(\boldsymbol{0}, \gamma(1-\eta)(1 - \gamma(1-\eta))\mI \right)$. Firstly, we present the FID results on the continuous CIFAR-10 dataset. We show the results with $50$ and $500$ NFE for efficiency. As shown in Tab.~\ref{tab:analysis_eta_continous}, excessively small or large values of $\eta$ detrimentally affect image generation quality. 

In addition, we found that SDE-based method (i.e., SDE-BFN-Solver++2) is less sensitive to $\eta$ compared to ODE-based methods (i.e., BFN-Solver++1, BFN-Solver++2). Assuming the discretization error is negligible, \citet{nie2023blessing} theoretically elucidates why SDE-based samplers outperform ODE-based samplers in sampling from the approximate prior distribution. This provides a theoretical foundation for our observations. 

At last, as present in Sec.~\ref{sec:exp_fast_continous} and Tab.~\ref{tab:fid_nfe}, after tuning $\eta$, ODE-based samplers still outperform SDE-based samplers and BFN baseline with a few NFEs (e.g., 10) under sample quality.

\begin{table}[t!]
\caption{\textbf{Empirical study of $\eta$.} We present the FID $\downarrow$ results of image generation on continuous CIFAR-10 dataset varying $\eta$.}
\label{tab:analysis_eta_continous}
\vskip 0.15in
\begin{center}
\begin{small}
\begin{sc}
\begin{tabular}{lccccccr}
\toprule
 & \multicolumn{3}{c}{NFE=50} & \multicolumn{3}{c}{NFE=500} \\
\cmidrule(lr){2-4} \cmidrule(lr){5-7} 
Values of  $\eta$          & 0.01        & 0.001       & 0.0001      & 0.001      & 0.0001   & 0.00001  \\ 
\midrule
BFN-Solver++1 (ours)        & 50.79       & 48.68       & 111.14      & 21.49      & 49.13    & 220.37   \\ 
BFN-Solver++2 (ours)        & 48.29       & 43.84       & 106.44      & 19.48      & 45.95    & 214.17   \\ 
SDE-BFN-Solver++2 (ours)    & 41.23       & 36.13       & 34.29       & 12.33      & 14.49    & 27.39     \\ 
\bottomrule
\end{tabular}
\end{sc}
\end{small}
\end{center}
\vskip -0.1in
\end{table}

\subsection{Analysis of $\eta$ on Discrete Data}
\label{app:analysis_eta_discrete}
In this section, we provide an analysis of hyperparameter $\eta$ for discrete data. According to the analyses in Appendix~\ref{sec:optimal-init}, we start sampling from the approximate prior distribution $\tilde{p}\left(\rvz(1 - \eta) \right) =  
 \mathcal{N}\left(\boldsymbol{0}, K\beta(1-\eta)\mI\right)$ which works well empirically. As shown in Tab.~\ref{tab:analysis_eta_dis}, For lower NFEs(20) or higher NFEs(200), the variation in $\eta$ shows minor differences in SA for BFN-Solver1, BFN-Solver2 and SDE-BFN-Solver2. This suggests that the choice of $\eta$ does not significantly impact the performance of these solvers on discrete data.

\begin{table}[t!]
\caption{\textbf{Empirical study of $\eta$.} We present the SA $\uparrow$ results of text generation on continuous text8 dataset varying $\eta$.}
\label{tab:analysis_eta_dis}
\vskip 0.15in
\begin{center}
\begin{small}
\begin{sc}
\begin{tabular}{lccccccr}
\toprule
 & \multicolumn{3}{c}{NFE=20} & \multicolumn{3}{c}{NFE=200} \\
\cmidrule(lr){2-4} \cmidrule(lr){5-7} 
Values of  $\eta$          & 0.01        & 0.001       & 0.0001      & 0.01      & 0.001   & 0.0001  \\ 
\midrule
BFN-Solver1 (ours)        & 82.56       & 82.39       & 82.46      & 85.40      & 85.38    & 85.39   \\ 
BFN-Solver2 (ours)        & 84.46       & 84.43       & 84.46      & 85.57      & 85.56    & 85.52   \\ 
SDE-BFN-Solver2 (ours)    & 84.98       & 85.28       & 85.27       & 86.65      & 86.68    & 86.74     \\ 
\bottomrule
\end{tabular}
\end{sc}
\end{small}
\end{center}
\vskip -0.1in
\end{table}

\subsection{Random Samples}
\label{app:random_samples}

Additional sampling results on CIFAR-10 are shown in Fig.~\ref{fig:samples_cifar}. 

Additional sampling results on text8 are shown in Figs.~\ref{fig:samples_text8_10} and ~\ref{fig:samples_text8_1000}.
\begin{figure}[ht]
\vskip 0.2in
\begin{center}
\centerline{\includegraphics[width=\columnwidth]{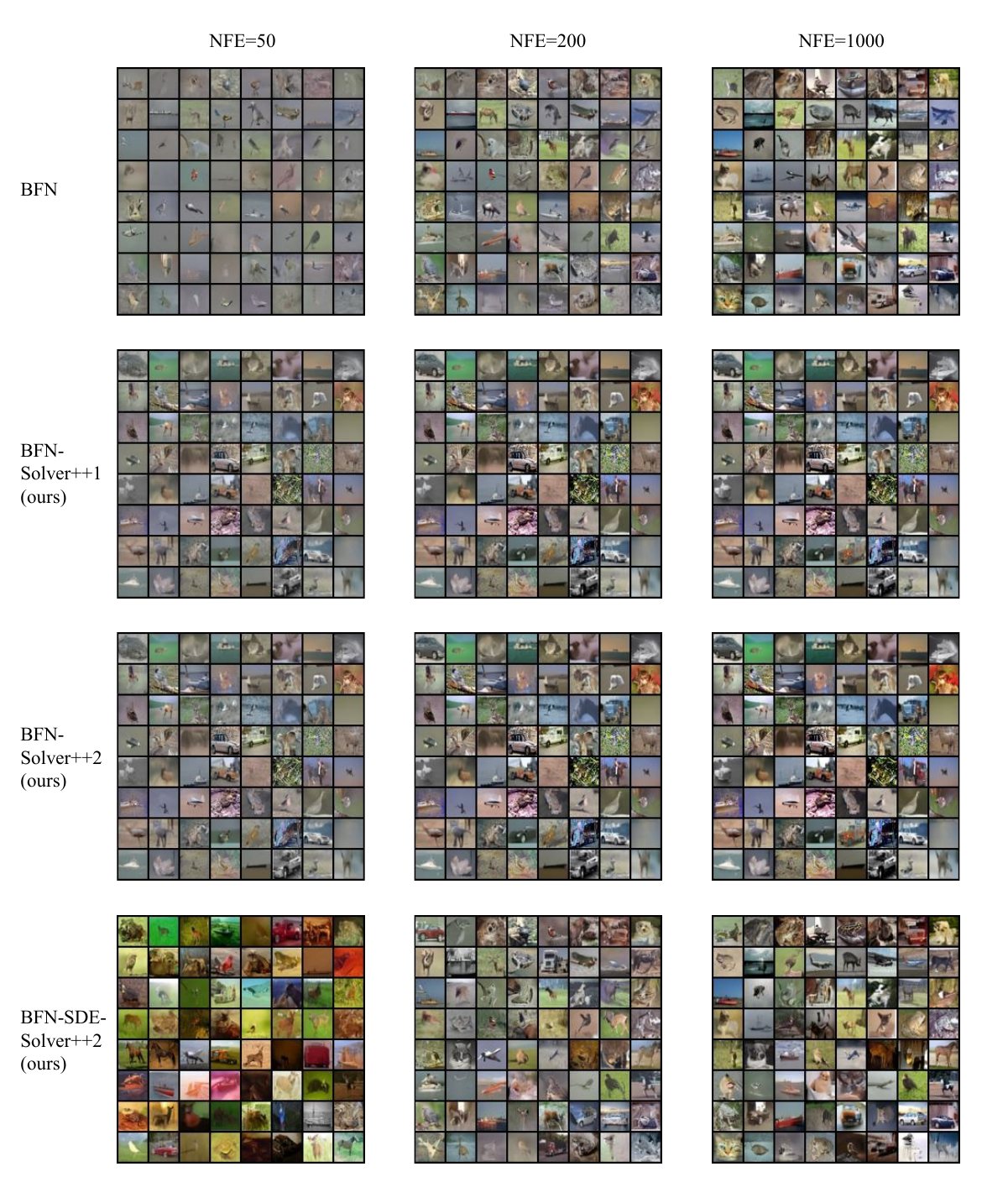}}
\caption{Randomly generated images by BFN and BFN-Solvers (ours) with \textbf{50}, \textbf{200}, \textbf{1000} NFEs, using the same pre-trained model on CIFAR-10.}
\label{fig:samples_cifar}
\end{center}
\vskip -0.2in
\end{figure}

\begin{figure}[ht]
\vskip 0.2in
\begin{center}
\centerline{\includegraphics[width=\columnwidth]{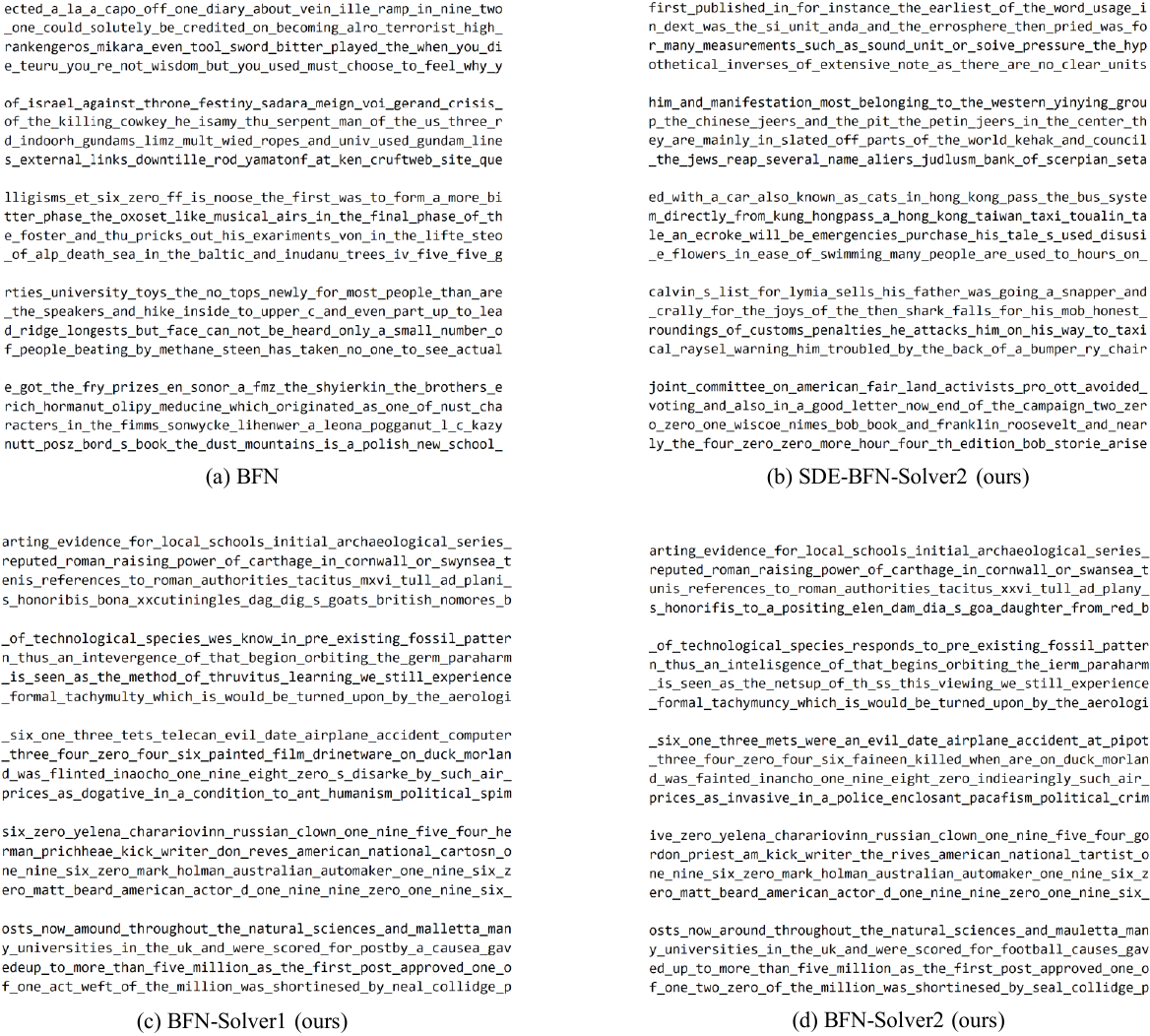}}
\caption{Randomly generated texts by BFN and BFN-Solvers (ours) with \textbf{10} NFEs, using the same pre-trained models on text8.}
\label{fig:samples_text8_10}
\end{center}
\vskip -0.2in
\end{figure}

\begin{figure}[ht]
\vskip 0.2in
\begin{center}
\centerline{\includegraphics[width=\columnwidth]{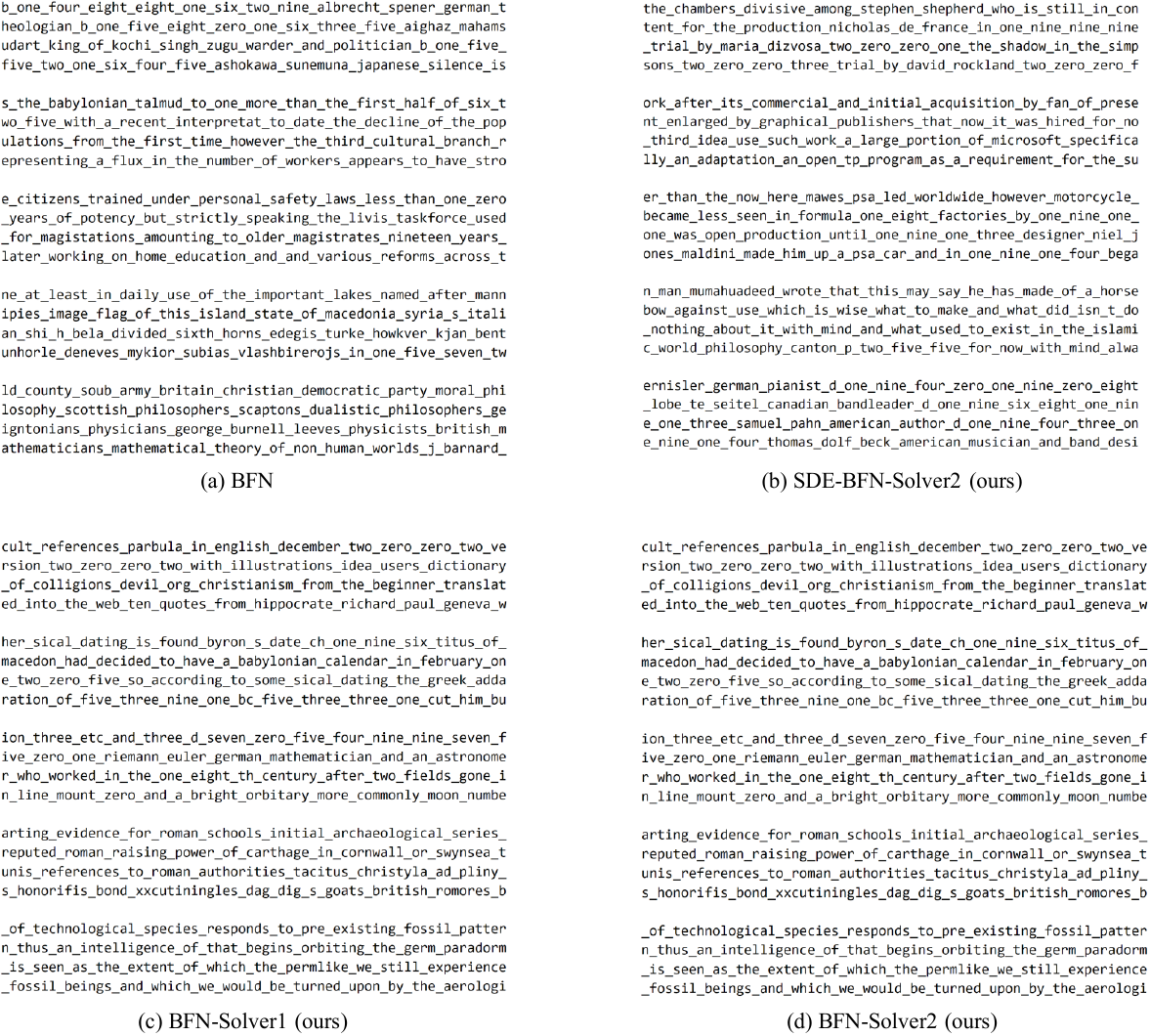}}
\caption{Randomly generated texts by BFN and BFN-Solvers (ours) with \textbf{1000} NFEs, using the same pre-trained models on text8.}
\label{fig:samples_text8_1000}
\end{center}
\vskip -0.2in
\end{figure}

\end{document}